%% file: iclr2027_conference.tex
\documentclass{article}
\usepackage{iclr2027_conference,times}

\usepackage{amsmath,amsthm}
\input{math_commands.tex}

\newtheorem{proposition}{Proposition}[section]
\newtheorem{corollary}{Corollary}[section]

\usepackage{graphicx}
\usepackage{booktabs,tabularx,colortbl}
\usepackage{float}
\usepackage{wrapfig}
\usepackage{needspace}
\usepackage{pifont}
\usepackage{enumitem}
\usepackage{hyperref}
\usepackage{url}
\definecolor{paperlink}{HTML}{4B4E8D}
\hypersetup{colorlinks=true, citecolor=paperlink, linkcolor=paperlink, urlcolor=paperlink}

\usepackage[skins]{tcolorbox}
\definecolor{questionbg}{HTML}{F3EFFA}
\definecolor{questionborder}{HTML}{5E6E7A}
\newtcolorbox{researchquestion}{
    enhanced,
    colback=questionbg, colframe=questionborder,
    boxrule=0.65pt, arc=8pt,
    left=51pt, right=11pt, top=4.5pt, bottom=4.5pt,
    before skip=8pt, after skip=9pt,
    fontupper=\itshape,
    overlay={
        \draw[questionborder!70, line width=0.45pt,
              dash pattern=on 2.4pt off 2.5pt]
            ([xshift=43pt,yshift=6pt]frame.south west) --
            ([xshift=43pt,yshift=-6pt]frame.north west);
        \node[inner sep=0pt]
            at ([xshift=23pt]frame.west)
            {\includegraphics[width=22pt]{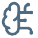}};
    }
}

\newcommand{\method}{{\fontfamily{lmtt}\mdseries\upshape\selectfont FutureWorlds}}
\newcommand{\memspo}{{\fontfamily{lmtt}\mdseries\upshape\selectfont MemSPO}}

\title{\raggedright \method{}: Learning Robotic World Models from Alternative Futures}

\author{%
\begin{minipage}[t]{\dimexpr\textwidth-2\tabcolsep\relax}
\raggedright\normalfont
{\bfseries
\mbox{Hao Wu\textsuperscript{\mdseries 1}},
\mbox{Shengju Qian\textsuperscript{\mdseries 2}},
\mbox{Weiyan Wang\textsuperscript{\mdseries 3,*}},
\mbox{Fan Xu\textsuperscript{\mdseries 4}},
\mbox{Fan Zhang\textsuperscript{\mdseries 2}},
\mbox{Yuanpeng He\textsuperscript{\mdseries 5}},
\mbox{Qingsong Wen\textsuperscript{\mdseries 6}},
\mbox{Yuxuan Liang\textsuperscript{\mdseries 1}\thanks{Corresponding authors.}}}\\[1pt]
\mbox{\textsuperscript{1}HKUST (GZ)}\quad
\mbox{\textsuperscript{2}CUHK}\quad
\mbox{\textsuperscript{3}Tencent}\quad
\mbox{\textsuperscript{4}USTC}\quad
\mbox{\textsuperscript{5}PKU}\quad
\mbox{\textsuperscript{6}Squirrel Ai Learning}
\end{minipage}%
}

\makeatletter
\patchcmd{\@maketitle}{\rule{\z@}{24pt}\@author}{\rule{\z@}{14pt}\@author}{}{\PackageError{futureworlds}{Title-author spacing patch failed}{}}
\patchcmd{\@maketitle}{\vskip 0.3in minus 0.1in}{\vskip 0.12in minus 0.02in}{}{\PackageError{futureworlds}{Author-abstract spacing patch failed}{}}
\makeatother

\iclrfinalcopy
\begin{document}
\raggedbottom

\maketitle
% Display authors without implying conference acceptance.
\lhead{Futureworlds}

\begin{abstract}
\input{sections/abstract}
\end{abstract}

\input{sections/introduction}
\input{sections/related_work}
\input{sections/method}
\input{sections/experiments}

\input{sections/conclusion}

\clearpage
\input{sections/statements}

\bibliography{iclr2027_conference}
\bibliographystyle{iclr2027_conference}

\clearpage
\appendix
\input{sections/appendix}

\end{document}

%% file: math_commands.tex
\usepackage{amsmath,amsfonts}
\usepackage{newtxmath}
\usepackage{bm}

\def\eqref#1{equation~\ref{#1}}
\def\1{\bm{1}}

\DeclareMathAlphabet{\mathsfit}{\encodingdefault}{\sfdefault}{m}{sl}
\SetMathAlphabet{\mathsfit}{bold}{\encodingdefault}{\sfdefault}{bx}{n}

%% file: sections/abstract.tex
Robotic world models predict action-conditioned future scenes, providing a foundation for understanding action outcomes.
However, turning alternative predictions into useful learning signals remains challenging: similar candidates limit informative quality comparisons, while diverging trajectories require persistent maintenance of their individual histories.
We introduce \method{}, a framework that unifies candidate construction, history maintenance, and learning from relative quality.
Built on a multimodal discrete autoregressive model, \method{} uses diverse beam search during reinforcement learning to construct candidate futures that balance confidence and diversity.
Candidate-specific bounded memory preserves scene states and ensures that generation and policy scoring use matching histories.
We further propose \memspo{} (Memory-Conditioned Search-Guided Policy Optimization), which converts video trajectory rewards into group-relative advantages to optimize the world model.
On RT-1, BridgeV2, and RoboCasa, \method{} reduces LPIPS for 32-frame predictions by 14.78\%, 20.84\%, and 9.12\%, respectively, relative to the strongest baseline on each dataset.
Under fixed evaluation configurations, only 200 \memspo{} updates further improve generation quality and support continued prediction beyond the training horizon.
Memory ablations, decoding sensitivity analysis, and optical-flow evaluation show that these gains extend beyond visual quality to more accurate motion prediction and more consistent object states.
Project page and code: \url{https://github.com/Alexander-wu/FutureWorlds}.

%% file: sections/introduction.tex
\section{Introduction}
\label{sec:introduction}

Robotic world models predict how scenes evolve in response to actions, providing a foundation for action-outcome prediction \citep{finn2016physical}, visual planning \citep{ebert2018visual,hafner2019planet}, behavior learning \citep{hafner2020dreamer}, and policy evaluation \citep{yin2026playworld,gigaworld2026}. Given the same observations, action sequence, and instruction, a model can propose multiple predictions of the future. These \emph{alternative futures} represent different hypotheses about the scene's evolution, with varying robot motions, object states, and prediction quality. Their differences offer a learning opportunity: by constructing useful alternatives and comparing their outcomes against recorded observations, a world model can learn from the predictions it generates.

% The teaser wraps into the actual introduction prose on the right.
% Only one input is displayed; all rollouts use two initial observations.
% Editable Python/SVG source: futureworlds_intro_20260920/build_figure.py.
% Allow the wrap figure to float to the next page without forcing an early text break.
% Let wrapfig calculate its height from the image and concise caption.

Most autoregressive world models learn from recorded trajectories through supervised training. Their reliance on ground-truth histories during training differs from the self-generated histories available at inference, allowing errors to accumulate along a rollout \citep{ross2011dagger,bengio2015sampling}. Token-level likelihood also does not directly optimize complete generated sequences \citep{ranzato2016sequence,rennie2017scst}. Figure~\ref{fig:intro_candidate_futures} illustrates this challenge: supervised fine-tuning (SFT) can leave predictions with insufficient robot motion or inaccurate poses. Reinforcement learning (RL) connects generated outputs to sequence-level rewards \citep{black2024ddpo}. In world modeling, RLVR-World uses decoded prediction metrics \citep{wu2025rlvr}, while RLIR introduces inverse-dynamics rewards for action following \citep{ye2025rlir}. These advances enable learning from trajectory feedback, but the resulting learning signal also depends on \emph{which futures the model compares}. More candidates alone do not guarantee useful comparative feedback. Effective post-training therefore requires informative candidate differences and history-consistent comparisons that turn alternative futures into reliable learning signals.

\needspace{300pt}
\begin{wrapfigure}{r}{0.53\textwidth}
    \vspace{-3pt}
    \setlength{\abovecaptionskip}{10.2pt}
    \centering
    \includegraphics[width=\linewidth]{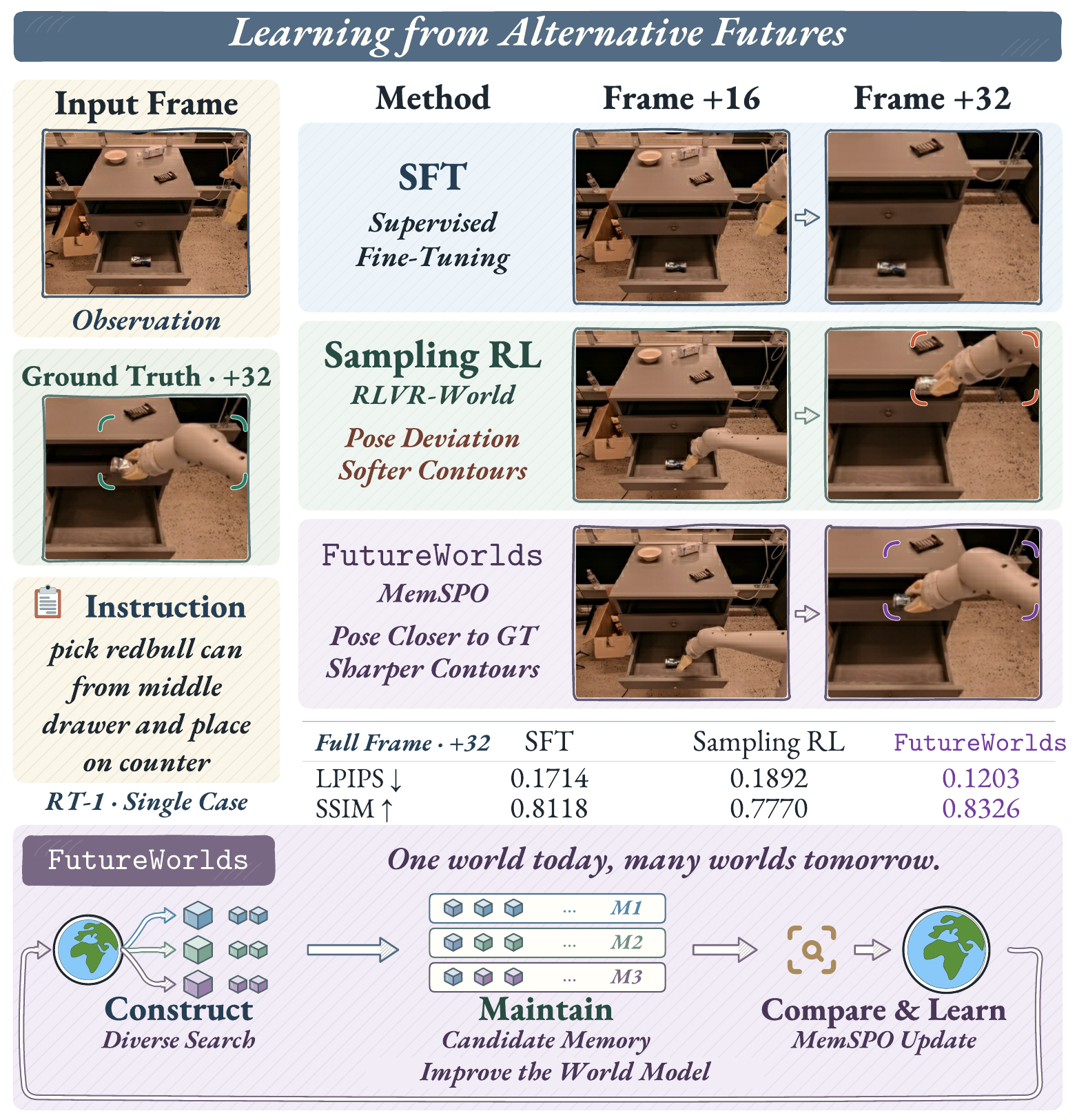}
    \caption{\textbf{\textit{Learning from alternative futures.}}
    \mbox{RT-1} comparison at $+32$; metrics use full images.
    \method{} constructs diverse futures, maintains their histories, and learns through \memspo{} by comparing trajectory rewards.}
    \label{fig:intro_candidate_futures}
\end{wrapfigure}

Learning from alternative futures requires coordinated candidate construction and history management. Candidate quality and diversity depend on the generation strategy \citep{holtzman2020degeneration}: independent sampling does not explicitly coordinate these properties across candidates, while conventional beam search can concentrate on nearly identical sequences \citep{vijayakumar2018diverse}. A useful candidate set therefore needs to balance model confidence with distinct predictions. As these predictions unfold, each branch develops its own robot and object states. Retaining the corresponding history becomes especially important when prediction extends beyond the pretraining horizon, motivating bounded memory for sustained generation \citep{dai2019transformerxl,wu2025spatialmemory,zheng2026memworld}. Learning introduces a further requirement: policy scoring must condition on the same retained history that produces each candidate. These considerations connect three roles---\emph{constructing} informative alternatives, \emph{maintaining} their individual histories, and \emph{learning} from their relative quality.

% End text wrapping before the full-width question box.
\par\WFclear
\vspace{-3pt}
\begin{researchquestion}
How can we construct informative alternative futures, maintain their individual histories, and learn from their relative quality to improve robotic world models?
\end{researchquestion}

We introduce \method{}, a framework that learns robotic world models by constructing, maintaining, and comparing alternative futures (Figure~\ref{fig:roboalign_m1_overview}). We instantiate the framework with a multimodal discrete autoregressive model \citep{oord2017vqvae,wu2024ivideogpt}, whose token probabilities support candidate expansion and trajectory scoring under visual, action, and language conditioning. \emph{To construct alternatives}, diverse beam search combines model likelihood with cross-candidate differences during RL training. \emph{To maintain their histories}, candidate-specific memory retains scene and state anchors alongside rolling context, supporting longer predictions within a bounded memory budget. Search and memory operate together as each candidate evolves, and generation and policy scoring use matching histories. \emph{To compare and learn}, \memspo{} (\underline{Mem}ory-Conditioned \underline{S}earch-Guided \underline{P}olicy \underline{O}ptimization) converts video trajectory rewards into group-relative advantages and updates the model through a search-guided surrogate based on GRPO \citep{shao2024deepseekmath} and clipped policy optimization \citep{schulman2017ppo}. Candidate quality differences thus become feedback for improving the model that generates them.

Our evaluation covers RT-1 \citep{brohan2023rt1}, BridgeData V2 \citep{walke2023bridge}, and RoboCasa \citep{nasiriany2026robocasa365}. On 32-frame predictions, \method{} reduces LPIPS by 14.78\%, 20.84\%, and 9.12\% relative to the strongest baseline on each dataset. Under fixed evaluation configurations, only 200 \memspo{} updates improve generation and support prediction beyond the training horizon. Memory ablations, decoding sensitivity analysis, and motion evaluation further assess historical context, robustness, and dynamic consistency. Our contributions are summarized as follows:
\begin{enumerate}[leftmargin=1.8em, labelsep=0.45em, topsep=4pt, itemsep=3pt, parsep=0pt]
    \item[\ding{182}] \emph{Learning from Alternative Futures.}
    We introduce \method{}, a framework that organizes model-generated futures into training experiences through coordinated candidate construction, history maintenance, and relative comparison.
    \item[\ding{183}] \emph{Search-Guided, History-Consistent Learning.}
    We propose \memspo{}, a GRPO-based post-training algorithm that combines diverse candidate search with candidate-specific memory and matches the histories used for generation and policy scoring.
    \item[\ding{184}] \emph{Experimental Evaluation.}
    We evaluate generation quality and long-horizon prediction, analyze memory and decoding robustness, and assess motion consistency and inference resource requirements. LPIPS decreases by 9.12--20.84\% versus the strongest baseline on each dataset.
\end{enumerate}

% End of introduction. No blank wrapfigure placeholder is needed.

%% file: sections/related_work.tex
\section{Related Work}
\label{sec:related-work}

\noindent\textbf{\emph{Robotic World Models.}}
Robotic world models predict action consequences for planning and policy learning. Early work uses action-conditioned video prediction for visual control \citep{finn2016physical,ebert2018visual}; PlaNet \citep{hafner2019planet} and Dreamer \citep{hafner2020dreamer} learn latent dynamics. UniSim \citep{yang2024unisim} combines language and low-level controls for visual simulation, while Vid2World \citep{huang2026vid2world} adapts video diffusion for interactive prediction. VideoGPT \citep{yan2021videogpt} and iVideoGPT \citep{wu2024ivideogpt} develop discrete autoregressive models, with iVideoGPT integrating observations, actions, and rewards through compressive tokenization. \method{} uses discrete autoregression to construct and compare alternative futures under shared visual, action, and language conditions.

\noindent\textbf{\emph{Memory for Long-Horizon World Modeling.}}
Transformer-XL \citep{dai2019transformerxl} propagates segment-level history, while StreamingLLM \citep{xiao2024streamingllm} retains initial attention sinks and recent context. StreamingT2V \citep{henschel2025streamingt2v} combines short- and long-term memory; FramePack \citep{zhang2025framepack} packs historical frames into a fixed context budget. \citet{wu2025spatialmemory} introduce geometry-grounded memory for scene revisits. Ctrl-World \citep{guo2026ctrlworld} and Mem-World \citep{zheng2026memworld} retrieve robotic observations through pose conditioning and spatiotemporal surfel indexing, respectively. \method{} maintains bounded, candidate-specific histories for long-horizon prediction and uses matching contexts for generation and policy scoring, preserving distinct conditioning as candidate futures diverge.

\noindent\textbf{\emph{Search-Guided Reinforcement Learning for World Models.}}
MIXER \citep{ranzato2016sequence} and SCST \citep{rennie2017scst} optimize sequence metrics through policy gradients; DDPO \citep{black2024ddpo} extends reward optimization to diffusion models. RLVR-World \citep{wu2025rlvr} rewards decoded prediction quality, while RLIR \citep{ye2025rlir} evaluates action following through inverse dynamics. Candidate construction determines which predictions provide this feedback. Diverse Beam Search \citep{vijayakumar2018diverse} reduces redundancy, Expert Iteration \citep{anthony2017expert} alternates search and policy learning, and BGRPO \citep{lee2025bgrpo} combines beam-generated candidates with rank-aware rewards for algebraic tasks. QWM \citep{dong2026qwm} searches predicted futures for robotic action selection and learns policies and values from real transitions. \method{} instead searches futures under a fixed action sequence. \memspo{} updates the video predictor from their relative quality using history-consistent scoring and a search-guided surrogate based on GRPO's group-relative advantages and clipped updates \citep{shao2024deepseekmath}.

%% file: sections/method.tex
\section{Method}
\label{sec:method}

Figure~\ref{fig:roboalign_m1_overview} presents \method{}, which learns from alternative futures through coordinated candidate construction, history maintenance, and relative comparison. A multimodal autoregressive world model first learns conditional video prediction through supervised training. During post-training, diverse search constructs candidate futures under shared visual, action, and language conditions. Each candidate retains its own history, with bounded persistent memory supporting continuation beyond the pretraining horizon. \memspo{} compares decoded trajectories through video rewards and converts their relative quality into world-model updates. Generation and policy scoring use matching candidate histories, while true future frames enter only reward computation during post-training.

\begin{figure}[!htbp]
    \centering
    \includegraphics[width=\linewidth]{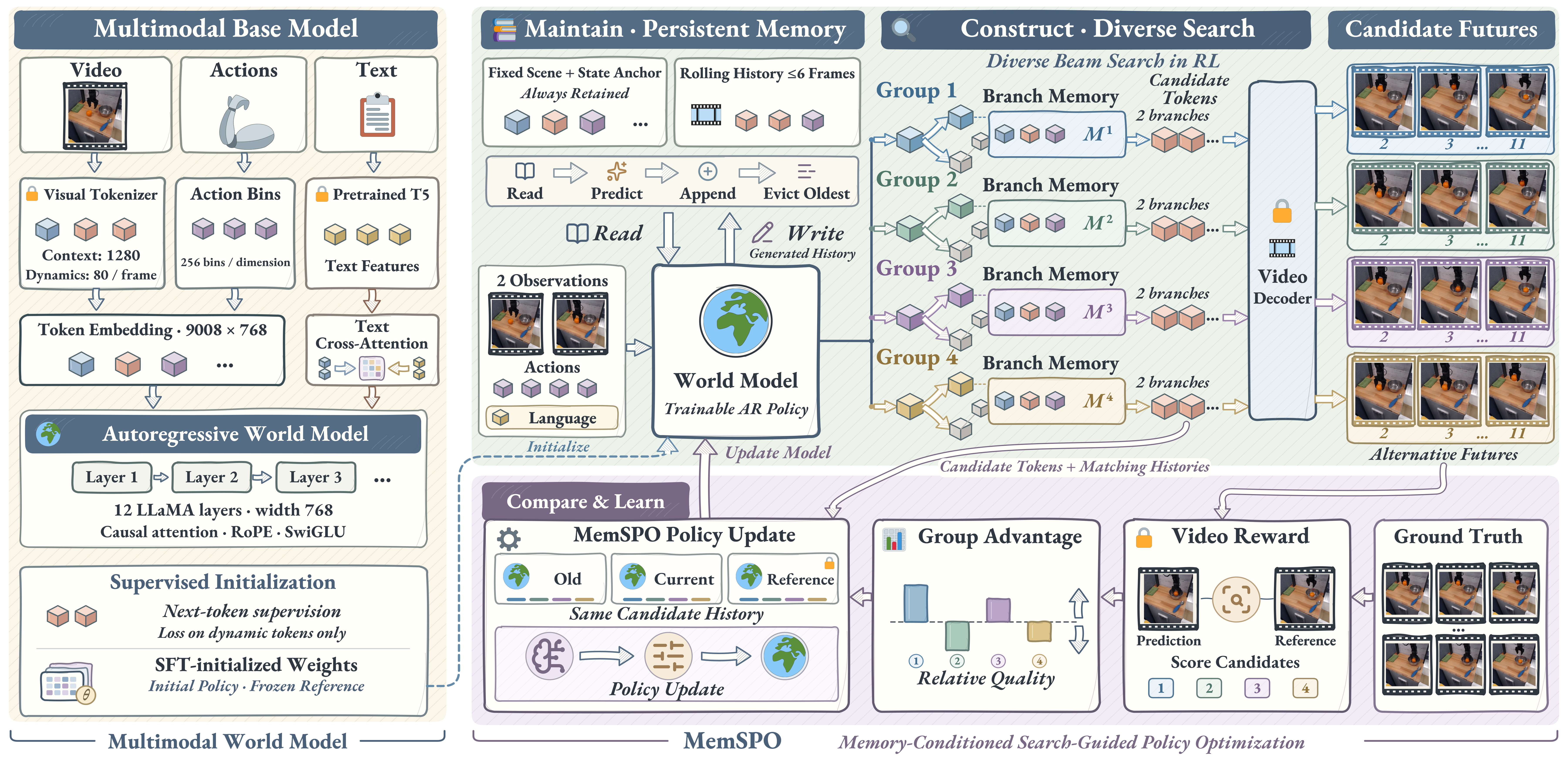}
    \caption{\textbf{\textit{Overview of \method{}.}}
    Supervised initialization trains the multimodal world model. \memspo{} searches alternative futures, maintains separate histories, and learns from relative trajectory quality. Generation and policy scoring use matching histories.}
    \label{fig:roboalign_m1_overview}
\end{figure}

% Allow the problem formulation to fill the preceding page naturally.
\subsection{Problem Formulation}
\label{sec:problem_formulation}

Robotic world modeling predicts how a scene evolves under actions. We use a GPT-style autoregressive model to represent alternative future trajectories conditioned on the same visual observations, action sequence, and optional language instruction. Let $\mathbf z_{1:H}$ denote the discrete visual sequence for $H$ future frames, with $N$ dynamic tokens per frame. Given initial visual conditions $\mathbf c_v$, robot actions $\mathbf a_{0:H-1}$, and an optional language instruction $\ell$, the world model factorizes the conditional distribution as
\begin{equation}
\pi_\theta(\mathbf z_{1:H}\mid\mathbf c_v,\mathbf a_{0:H-1},\ell)
=\prod_{t=1}^{H}\prod_{j=1}^{N}
\pi_\theta(z_{t,j}\mid\mathbf c_v,\mathbf z_{<t},\mathbf z_{t,<j},\mathbf a_{<t},\ell).
\label{eq:autoregressive_world_model}
\end{equation}
Here, $\mathbf z_{<t}$ contains tokens from preceding frames, $\mathbf z_{t,<j}$ is the within-frame prefix, and $\mathbf a_{<t}=\mathbf a_{0:t-1}$ contains actions available before frame $t$. In particular, $a_{t-1}$ conditions the transition to that frame. A visual decoder maps predicted tokens back to video. This factorization supports both candidate expansion and trajectory scoring. Our goal is to construct alternative predictions under the same conditions, maintain their histories within a bounded context, and learn from their relative quality.

\subsection{Multimodal Autoregressive World Model}
\label{sec:multimodal_world_model}

The multimodal model provides the representation for constructing and scoring candidate futures. Following iVideoGPT's conditional video compression~\citep{wu2024ivideogpt}, we use more context tokens for the initial scene and fewer dynamic tokens for subsequent changes, reducing repeated encoding of static information. Context and conditional dynamics encoders produce discrete representations through Finite Scalar Quantization (FSQ; \citealp{mentzer2024fsq}), and a visual decoder reconstructs frames:
\begin{equation}
C=Q_c\!\left(E_\phi^c(o_{\mathrm{ref}})\right),\qquad
\mathbf z_t=Q_d\!\left(E_\phi^d(o_t\mid o_{\mathrm{ref}})\right),\qquad
\tilde o_t=D_\phi(C,\mathbf z_t).
\label{eq:tokenizer_reconstruction}
\end{equation}
Here, $o_{\mathrm{ref}}$ is the initial reference frame, $\phi$ denotes tokenizer parameters, and $Q_c$ and $Q_d$ map continuous features to discrete tokens with straight-through gradients.

The tokenizer jointly optimizes pixel, perceptual, and adversarial losses:
\begin{equation}
\mathcal L_{\mathrm{tok}}(\phi)
=\lambda_{\mathrm{rec}}\mathcal L_1
+\lambda_{\mathrm{perc}}\mathcal L_{\mathrm{LPIPS}}
+\lambda_{\mathrm{adv}}\mathcal L_{\mathrm{GAN}}.
\label{eq:tokenizer_objective}
\end{equation}
Pixel and perceptual losses supervise both reference and dynamic frames; the adversarial term further constrains visual quality. The tokenizer remains frozen throughout world-model SFT and RL. Dataset-specific initialization and adaptation details are given in Appendix~\ref{app:tokenizer-details}.
% Discriminator implementation: hinge loss with gradient penalty.

As Figure~\ref{fig:roboalign_m1_overview} (left) shows, actions undergo per-dimension normalization and discretization, then form a temporally ordered causal sequence with visual tokens. A frozen T5 encoder~\citep{raffel2020t5} supplies language features through cross-attention. Supervised fine-tuning conditions on ground-truth histories and applies negative log-likelihood only to target dynamic tokens:
\begin{equation}
\mathcal L_{\mathrm{SFT}}(\theta)
=-\mathbb E_{\mathcal D}\!\left[
\frac{1}{HN}\sum_{t=1}^{H}\sum_{j=1}^{N}
\log\pi_\theta(z_{t,j}\mid\mathbf c_v,\mathbf z_{<t},\mathbf z_{t,<j},\mathbf a_{<t},\ell)
\right].
\label{eq:sft_objective}
\end{equation}
This objective jointly updates the autoregressive backbone, action embeddings, and text cross-attention modules. The resulting model initializes the post-training policy, with a frozen copy serving as the reference policy.

\subsection{Candidate-Specific Persistent Memory}
\label{sec:persistent_memory}

As candidate futures diverge, each accumulates its own prediction history. Persistent memory supports long-horizon continuation within a bounded context by retaining the initial scene, a state anchor, and recent predictions for each candidate. Let $\mathbf c_v=(C,\mathbf z_{\mathrm{anc}})$ and let $\tilde{\mathbf a}_s$ denote discrete action tokens. Before predicting frame $t$, candidate $i$ has memory
\begin{equation}
M_t^{(i)}
=\left[C,\mathbf z_{\mathrm{anc}},\tilde{\mathbf a}_0,
\left\{(\hat{\mathbf z}_s^{(i)},\tilde{\mathbf a}_s)\right\}_{s\in\mathcal I_t}\right],
\quad \mathcal I_t=\{\max(1,t-W),\ldots,t-1\}.
\label{eq:candidate_memory}
\end{equation}
The rolling history is temporally ordered and initially empty; its window length $W$ satisfies the token budget $|M_t^{(i)}|\le B$. Scene and state anchors persist throughout prediction. After each frame, memory appends its dynamic tokens and corresponding action, excluding older history from the model context once the window is full. Language remains available independently of this history.

During search, branches inherit their parents' memories, which are reordered together with candidate selection. To learn from each resulting future under the conditions that generate it, policy scoring uses the same memory and within-frame prefix:
\begin{equation}
\log\pi_\theta(\hat{\mathbf z}_{1:H}^{(i)}\mid\mathbf c_v,\mathbf a_{0:H-1},\ell)
=\sum_{t=1}^{H}\sum_{j=1}^{N}
\log\pi_\theta(\hat z_{t,j}^{(i)}\mid M_t^{(i)},\hat{\mathbf z}_{t,<j}^{(i)},\ell).
\label{eq:memory_conditioned_scoring}
\end{equation}
The old, current, and reference policies share these conditions, preserving the link between each candidate and its history during optimization. Supervised training uses ground-truth histories; RL and free rollouts use predicted histories, with true future frames reserved for reward computation.

\subsection{Memory-Conditioned Search-Guided Policy Optimization}
\label{sec:memspo}

\memspo{} turns alternative futures into post-training feedback: it constructs candidates through diverse search, maintains their histories as described above, and learns from their relative trajectory quality. The update uses GRPO-style group-relative advantages and a clipped policy surrogate~\citep{shao2024deepseekmath,schulman2017ppo}.

\noindent\textbf{\emph{Constructing alternative futures.}}
We maintain $G$ search groups with $b$ beams each. For prefix $y$ and next token $v$, group $g$ uses the diverse search score~\citep{vijayakumar2018diverse}:
\begin{equation}
S_g(y\mathbin{\oplus}v)
=S_g(y)+\log\pi_{\mathrm{old}}(v\mid u(y))
-\lambda_{\mathrm{div}}d_g(v).
\label{eq:diverse_search}
\end{equation}
Here, $u(y)$ contains the branch's memory and multimodal conditions, and $d_g(v)$ measures the normalized frequency of $v$ selected by earlier groups at the same position. Each group retains its $b$ highest-scoring branches and returns one final trajectory, producing $G$ videos under the same input conditions. The penalty encourages token-level differences across groups and affects selection only; policy scoring uses unmodified probabilities.

\noindent\textbf{\emph{Comparing trajectory quality.}}
The frozen visual decoder converts candidate tokens into videos. Following RLVR-World~\citep{wu2025rlvr}, our base reward combines pixel error and perceptual distance~\citep{zhang2018lpips}:
\begin{equation}
R_i=-\frac{1}{H}\sum_{t=1}^{H}
\left[\operatorname{MAE}(\hat o_t^{(i)},o_t)
+\operatorname{LPIPS}(\hat o_t^{(i)},o_t)\right].
\label{eq:trajectory_reward}
\end{equation}
For either reward design (Appendix~\ref{app:reward-design}), we normalize scores within each input:
\begin{equation}
A_i=\frac{R_i-\bar R}{\operatorname{Std}(R_1,\ldots,R_G)+\varepsilon},
\qquad \bar R=\frac{1}{G}\sum_{i=1}^{G}R_i.
\label{eq:relative_advantage}
\end{equation}
All dynamic tokens in a trajectory share $A_i$. True future frames enter reward computation only, never candidate memory. Video decoding and reward computation carry no gradients.

\noindent\textbf{\emph{Learning under matching histories.}}
The current, old, and frozen reference policies share each candidate's generation condition $u_{t,j}^{(i)}=(M_t^{(i)},\hat{\mathbf z}_{t,<j}^{(i)},\ell)$. The token probability ratio is
\begin{equation}
\rho_{t,j}^{(i)}(\theta)
=\frac{\pi_\theta(\hat z_{t,j}^{(i)}\mid u_{t,j}^{(i)})}
{\pi_{\mathrm{old}}(\hat z_{t,j}^{(i)}\mid u_{t,j}^{(i)})}.
\label{eq:policy_ratio}
\end{equation}
\memspo{} maximizes the surrogate objective
\begin{equation}
\mathcal J_{\mathrm{MemSPO}}(\theta)
=\frac{1}{GHN}\sum_{i=1}^{G}\sum_{t=1}^{H}\sum_{j=1}^{N}
\left[\ell_{\mathrm{clip}}\!\left(\rho_{t,j}^{(i)}(\theta),A_i\right)
-\beta\kappa_{t,j}^{(i)}(\theta)\right].
\label{eq:memspo_objective}
\end{equation}
Here, $\ell_{\mathrm{clip}}$ is the clipped policy term, including dual clipping for negative advantages. The reference regularizer is $\kappa=q-\log q-1$, where $q=\pi_{\mathrm{ref}}/\pi_\theta$ evaluates the same token under the same history.

Each iteration constructs fresh candidates and performs one world-model update, transferring their relative-quality feedback into the model parameters. The visual tokenizer, text encoder, and memory rules remain fixed. Beam selection changes the candidate distribution, so this is a search-guided surrogate without an unbiased on-policy gradient guarantee. We analyze how candidate construction shapes the group-relative reward gradient and bound its magnitude under candidate collapse (Appendix~\ref{app:theory}, Proposition~\ref{prop:pairwise-update} and Corollary~\ref{cor:candidate-collapse}).
% Implementation detail retained for later appendix expansion:
% v=min{rho*A,clip(rho,1-epsilon,1+epsilon)*A}; for A<0,
% the inherited dual-clipping term is max{v,c_dual*A}, c_dual>1.
% The reference penalty uses numerical clipping; negligible reward variance
% yields zero advantages. One update is applied per fresh candidate set;
% no gain is attributed to clipping being active in this configuration.

%% file: sections/experiments.tex
\section{Experiments}
\label{sec:experiments}

% Compact, consistent figure/table spacing for experiments and the appendix.
\setlength{\textfloatsep}{10pt plus 2pt minus 2pt}
\setlength{\floatsep}{8pt plus 2pt minus 2pt}
\setlength{\intextsep}{8pt plus 2pt minus 2pt}

Our experiments address four research questions:
\mbox{\ding{182}\,\textbf{\textit{RQ1}}}~How well does FutureWorlds predict future observations?
\mbox{\ding{183}\,\textbf{\textit{RQ2}}}~Can MemSPO improve generation with few post-training updates?
\mbox{\ding{184}\,\textbf{\textit{RQ3}}}~Do generation gains translate into more accurate motion and dynamics?
\mbox{\ding{185}\,\textbf{\textit{RQ4}}}~Does memory maintenance improve prediction beyond the training horizon?

\subsection{Experimental Setup}
\label{sec:experimental-setup}

\noindent\textbf{\textit{Baselines.}}
We compare four baseline categories: \mbox{\ding{70}}~video generation models: FitVid~\citep{babaeizadeh2021fitvid} and iVideoGPT~\citep{wu2024ivideogpt}; \mbox{\ding{70}}~action-conditioned world models: IRASim~\citep{zhu2025irasim}, Vid2World~\citep{huang2026vid2world}, Ctrl-World~\citep{guo2026ctrlworld}, DriftWorld~\citep{lu2026driftworld}, and WEAVER~\citep{jain2026weaver}; \mbox{\ding{70}}~RL post-training methods: RLVR-World~\citep{wu2025rlvr} and PersistWorld~\citep{bardhan2026persistworld}; and \mbox{\ding{70}}~adapted video foundation models: Wan2.2-TI2V-5B~\citep{wan2025wan}, Cosmos-Predict2.5-2B~\citep{nvidia2025cosmos25}, and DreamDojo-2B~\citep{gao2026dreamdojo}. Adaptation and evaluation details appear in Appendix~\ref{app:baseline-adaptation}.

\noindent\textbf{\textit{Datasets and Evaluation.}}
We evaluate action-conditioned video generation on RT-1~\citep{brohan2023rt1}, BridgeV2~\citep{walke2023bridge}, and RoboCasa~\citep{nasiriany2026robocasa365}, covering real and simulated manipulation. Table~\ref{tab:main-results} compares 32-frame predictions on 128 fixed held-out trajectories per dataset using PSNR, SSIM, and LPIPS~\citep{zhang2018lpips}. Methods share trajectories and target frames, with predictions resized to the reference resolution. Metric definitions and evaluation details appear in Appendix~\ref{app:baseline-adaptation}.

% Present the main quantitative and qualitative comparison together near the start of experiments.
\begin{figure}[!t]
\begin{minipage}{\linewidth}
\setlength{\intextsep}{4pt}
\input{tables/main_results}
\input{sections/rq1_qualitative}
\end{minipage}
\end{figure}

\noindent\textbf{\textit{Implementation Details.}}
We train \method{} in PyTorch on 16 NVIDIA H20 GPUs across two servers. Starting from SFT, \memspo{} uses AdamW with learning rate $10^{-6}$ and freezes the visual tokenizer and text encoder. Diverse beam search uses four groups of two beams, retaining one trajectory per group for reward comparison and updates. Main evaluations use a beam width of four and persistent memory. Pretrained initialization, SFT, and RL settings appear in Appendix~\ref{app:training-settings}.

\input{sections/rq1_results}

% Keep Table 2 and Figure 4 together in the next available top position.
\begin{figure}[!ht]
\begin{minipage}{\linewidth}
\begingroup
\setlength{\intextsep}{5pt}
\input{tables/matched_post_training}
\begin{figure}[H]
    \centering
    \setlength{\parskip}{0pt}
    \includegraphics[width=\linewidth]{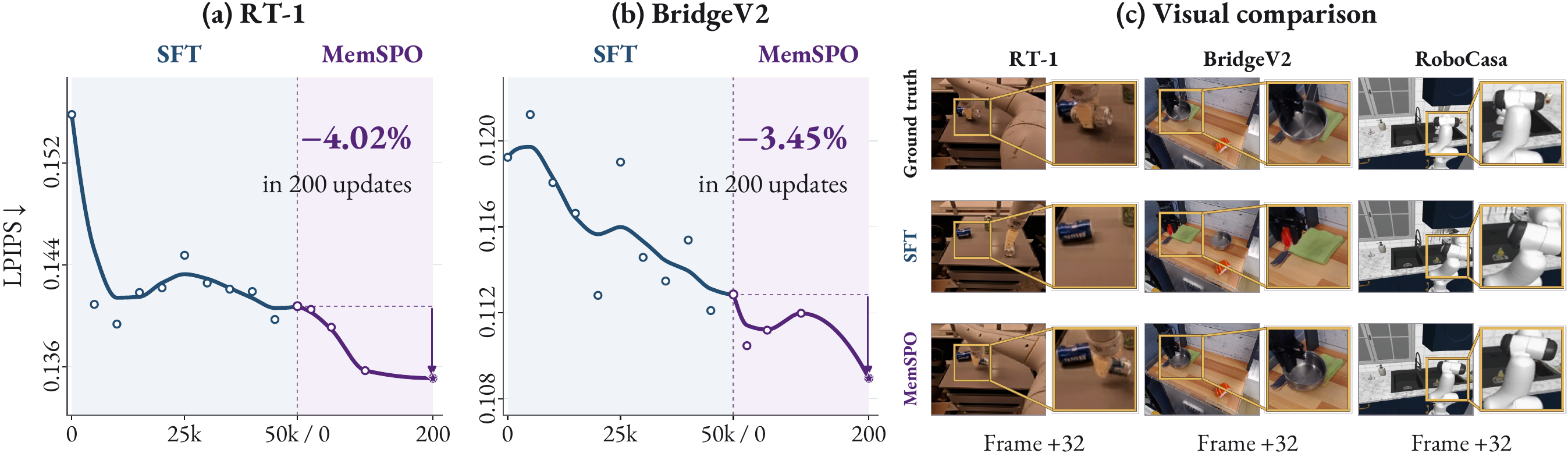}\par
    \vspace{3pt}
    \setlength{\abovecaptionskip}{2pt}
    \caption{\textbf{\textit{Post-training gains and qualitative improvements.}}
    (a,b) LPIPS during SFT and MemSPO training. (c) Frame-32 predictions with matched detail views on three datasets.}
    \label{fig:rq4-learning}
\end{figure}
\endgroup
\end{minipage}
\end{figure}

\input{sections/rq4_results}

\input{sections/motion_results}

\input{sections/memory_results}

%% file: tables/main_results.tex
% H32 results; RoboCasa Vid2World/DriftWorld/RLVR-World supplied by user 2026-09-23; provenance: main_results_sources.json.
% Do not insert tokenizer reconstruction scores into this forecasting table.
% Filled entries predict 32 future frames over 128 trajectories.
% All reported methods within a dataset must use the same held-out cohort.
\begingroup
\definecolor{resultink}{HTML}{292C42}
\definecolor{resultpurple}{HTML}{544178}
\definecolor{resultheader}{HTML}{E9E8F2}
\definecolor{resultstripe}{HTML}{F6F7FA}
\definecolor{resultours}{HTML}{EEE7F7}
\definecolor{rankfirst}{HTML}{E4D9F1}
\definecolor{ranksecond}{HTML}{F7E7D6}
\definecolor{rankthird}{HTML}{E0EDF4}
\newcommand{\rankone}[1]{\cellcolor{rankfirst}#1}
\newcommand{\ranktwo}[1]{\cellcolor{ranksecond}#1}
\newcommand{\rankthree}[1]{\cellcolor{rankthird}#1}
\definecolor{resultpending}{HTML}{7A7D8C}
\newcommand{\pendingresult}{\textcolor{resultpending}{\textendash}}
\newcommand{\pendingdataset}{\pendingresult & \pendingresult & \pendingresult}
\newcommand{\pendingrow}{\pendingdataset & \pendingdataset & \pendingdataset}
\newcommand{\resultcategory}[1]{%
  \addlinespace[1pt]
  \multicolumn{10}{@{}l}{\rule{0pt}{10.5pt}\textcolor{resultpurple}{\bfseries\itshape #1}}\\[0.5pt]}
\newcommand{\resultmetrics}{%
  {\fontsize{8.5}{9.6}\selectfont\mbox{PSNR}} &
  {\fontsize{8.5}{9.6}\selectfont\mbox{SSIM}} &
  {\fontsize{8.5}{9.6}\selectfont\mbox{LPIPS}}}

\begin{table}[H]
\centering
\caption{Quantitative comparison of 32-frame action-conditioned video prediction (128 trajectories per dataset).
\mbox{Metrics:} PSNR$\uparrow$, SSIM$\uparrow$, and LPIPS$\downarrow$.
\colorbox{rankfirst}{1st}, \colorbox{ranksecond}{2nd}, and \colorbox{rankthird}{3rd} mark the top three scores; the last row reports relative gains over the runner-up.}
\label{tab:main-results}
\vspace{3pt}
% Local Garamond family; tabular lining figures keep future results aligned.
\fontencoding{T1}\fontfamily{EBGaramond-TLF}\selectfont
\fontsize{9.0}{10.5}\selectfont
\color{resultink}
\setlength{\tabcolsep}{1.2pt}
\renewcommand{\arraystretch}{1.06}
\setlength{\aboverulesep}{2pt}
\setlength{\belowrulesep}{2pt}
% Only method names determine the first-column width; categories span all 10 columns.
\begin{tabularx}{\linewidth}{@{}p{0.25\linewidth}*{9}{>{\centering\arraybackslash}X}@{}}
\toprule[0.95pt]
\rowcolor{resultheader}
\rule{0pt}{10.5pt}\textbf{Method}
 & \multicolumn{3}{c}{\textbf{RT-1}}
 & \multicolumn{3}{c}{\textbf{BridgeV2}}
 & \multicolumn{3}{c@{}}{\textbf{RoboCasa}}\\[-0.5pt]
\cmidrule(lr){2-4}\cmidrule(lr){5-7}\cmidrule(l){8-10}
 & \resultmetrics & \resultmetrics & \resultmetrics\\[-0.5pt]
\midrule[0.55pt]
\resultcategory{Video Generation Models}
\rowcolor{resultstripe} FitVid & 18.71 & 0.5916 & 0.4750 & 17.09 & 0.6072 & 0.5235 & 13.50 & 0.5422 & 0.5655\\
iVideoGPT & 15.96 & 0.5571 & 0.3833 & \rankthree{20.06} & \rankthree{0.7791} & 0.1899 & \ranktwo{16.67} & \rankthree{0.7187} & 0.2366\\
\resultcategory{Action-Conditioned World Models}
\rowcolor{resultstripe} IRASim & 16.83 & 0.5936 & 0.3474 & 17.39 & 0.7094 & 0.2524 & 13.55 & 0.5834 & 0.3706\\
Vid2World & 19.03 & 0.7210 & \rankthree{0.1960} & 18.92 & 0.7110 & 0.1880 & 15.40 & 0.7000 & \rankthree{0.2287}\\
\rowcolor{resultstripe} Ctrl-World & 19.11 & 0.7110 & 0.2522 & 18.38 & 0.6932 & 0.2621 & 15.55 & 0.7035 & 0.2686\\
DriftWorld & \ranktwo{20.54} & \rankthree{0.7410} & 0.2595 & \ranktwo{20.32} & 0.7429 & 0.2379 & \rankthree{16.50} & \ranktwo{0.7200} & 0.2891\\
\rowcolor{resultstripe} WEAVER & \rankthree{19.74} & \ranktwo{0.7617} & 0.2321 & 18.74 & 0.7447 & 0.2476 & 14.62 & 0.7136 & 0.2726\\
\resultcategory{RL Post-Training}
\rowcolor{resultstripe} RLVR-World & 19.47 & 0.7360 & \ranktwo{0.1820} & 19.06 & 0.7210 & \ranktwo{0.1780} & 15.80 & 0.7100 & \ranktwo{0.2017}\\
PersistWorld & 18.77 & 0.6995 & 0.2717 & 19.07 & 0.7154 & 0.2408 & 15.61 & 0.7112 & 0.2639\\
\resultcategory{Adapted Video Foundation Models}
\rowcolor{resultstripe} Wan2.2-TI2V-5B & 18.42 & 0.7058 & 0.2478 & 16.35 & 0.6889 & 0.2568 & 14.37 & 0.6906 & 0.2711\\
Cosmos-Predict2.5-2B & 17.77 & 0.6765 & 0.2889 & 18.98 & \ranktwo{0.7799} & \rankthree{0.1870} & 14.30 & 0.6802 & 0.2916\\
\rowcolor{resultstripe} DreamDojo-2B & 18.47 & 0.7152 & 0.2660 & 16.95 & 0.7461 & 0.2563 & 13.92 & 0.6349 & 0.3278\\
\addlinespace[2pt]
\midrule[0.6pt]
\rowcolor{resultours}
\rule{0pt}{12pt}\textcolor{resultpurple}{\bfseries FutureWorlds (ours)} & \rankone{\textbf{22.47}} & \rankone{\textbf{0.8026}} & \rankone{\textbf{0.1551}} & \rankone{\textbf{22.03}} & \rankone{\textbf{0.7968}} & \rankone{\textbf{0.1409}} & \rankone{\textbf{18.47}} & \rankone{\textbf{0.7659}} & \rankone{\textbf{0.1833}}\\[1pt]
\addlinespace[1pt]
\textcolor{resultpurple}{\itshape Gain over runner-up} & \textcolor{resultpurple}{+9.40\%} & \textcolor{resultpurple}{+5.37\%} & \textcolor{resultpurple}{+14.78\%} & \textcolor{resultpurple}{+8.42\%} & \textcolor{resultpurple}{+2.17\%} & \textcolor{resultpurple}{+20.84\%} & \textcolor{resultpurple}{+10.80\%} & \textcolor{resultpurple}{+6.38\%} & \textcolor{resultpurple}{+9.12\%}\\[1pt]
\bottomrule[0.95pt]
\end{tabularx}
\vspace{3.5pt}
\end{table}
\endgroup

%% file: sections/rq1_qualitative.tex
\begingroup
\setlength{\intextsep}{6pt}
\begin{figure}[H]
    \centering
    \includegraphics[width=\linewidth]{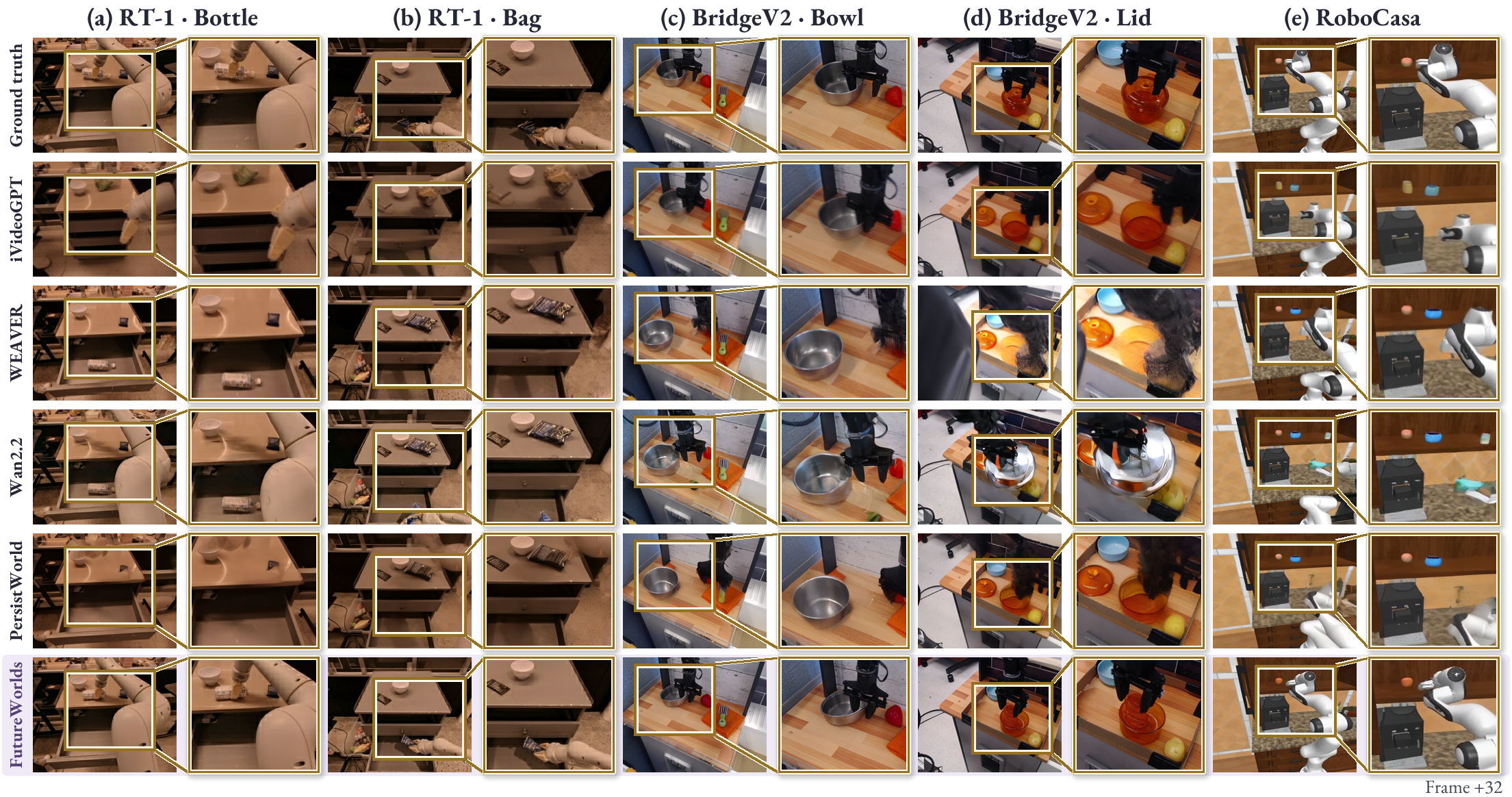}\par
    \setlength{\abovecaptionskip}{3pt}
    \caption{\textbf{\textit{Qualitative comparison at frame 32.}} Five selected cases on RT-1, BridgeV2, and RoboCasa comparing four baselines and our method, with full images and matched detail views.}
    \label{fig:rq1-qualitative}
\end{figure}
\endgroup

%% file: sections/rq1_results.tex
\subsection{Video Generation Quality (RQ1)}
\label{sec:rq1-quality}

In Table~\ref{tab:main-results}, \method{} leads all three metrics on all datasets. Relative to the strongest baseline per metric on BridgeV2, it gains 1.71~dB PSNR and 2.17\% SSIM and reduces LPIPS by 20.84\%. LPIPS also decreases by 14.78\% on RT-1 and 9.12\% on RoboCasa relative to the strongest baseline per dataset. Figure~\ref{fig:rq1-qualitative} shows closer object locations and arm poses in RT-1, fewer object substitutions and local distortions in BridgeV2, and more accurate arm and gripper poses in RoboCasa.

These comparisons support improved object consistency and action-conditioned prediction. The matched crops expose incorrect object positions relative to the gripper despite plausible textures, distinguishing faithful action-conditioned prediction from visually plausible but inconsistent manipulation states. Further quantitative and qualitative results appear in the appendix. A qualitative case in Appendix~\ref{app:policy-feedback} illustrates how predicted futures provide feedback for action selection toward closed-loop control.

%% file: tables/matched_post_training.tex
\begingroup
\definecolor{resultink}{HTML}{292C42}
\definecolor{resultpurple}{HTML}{544178}
\definecolor{resultheader}{HTML}{E9E8F2}
\definecolor{resultstripe}{HTML}{F6F7FA}
\definecolor{resultours}{HTML}{EEE7F7}
\definecolor{rankfirst}{HTML}{E4D9F1}
\definecolor{ranksecond}{HTML}{F7E7D6}
\definecolor{rankthird}{HTML}{E0EDF4}
\newcommand{\rankone}[1]{\cellcolor{rankfirst}#1}
\newcommand{\ranktwo}[1]{\cellcolor{ranksecond}#1}
\newcommand{\rankthree}[1]{\cellcolor{rankthird}#1}
\definecolor{resultpending}{HTML}{7A7D8C}
\newcommand{\pendingresult}{\textcolor{resultpending}{\textendash}}
\newcommand{\pendingdataset}{\pendingresult & \pendingresult & \pendingresult}
\newcommand{\pendingrow}{\pendingdataset & \pendingdataset & \pendingdataset}
\newcommand{\resultcategory}[1]{%
  \addlinespace[0.5pt]
  \multicolumn{10}{@{}l}{\rule{0pt}{9.5pt}\textcolor{resultpurple}{\bfseries\itshape #1}}\\[0pt]}
\newcommand{\resultmetrics}{%
  {\fontsize{8.5}{9.6}\selectfont\mbox{PSNR}} &
  {\fontsize{8.5}{9.6}\selectfont\mbox{SSIM}} &
  {\fontsize{8.5}{9.6}\selectfont\mbox{LPIPS}}}

\begin{table}[H]
\centering
\caption{Post-training comparison with 200 updates. Metrics: PSNR$\uparrow$, SSIM$\uparrow$, and LPIPS$\downarrow$. Gain rows show relative improvements over GRPO (\%).}
\label{tab:matched-post-training}
\vspace{1pt}
\fontencoding{T1}\fontfamily{EBGaramond-TLF}\selectfont
\fontsize{9.0}{10.5}\selectfont
\color{resultink}
\setlength{\tabcolsep}{1.2pt}
\renewcommand{\arraystretch}{0.96}
\setlength{\aboverulesep}{1.6pt}
\setlength{\belowrulesep}{1.6pt}
% Only method names determine the first-column width; categories span all 10 columns.
\begin{tabularx}{\linewidth}{@{}p{0.25\linewidth}*{9}{>{\centering\arraybackslash}X}@{}}
\toprule[0.95pt]
\rowcolor{resultheader}
\rule{0pt}{9.5pt}\textbf{Method}
 & \multicolumn{3}{c}{\textbf{10 frames}}
 & \multicolumn{3}{c}{\textbf{20 frames}}
 & \multicolumn{3}{c@{}}{\textbf{32 frames}}\\[-0.5pt]
\cmidrule(lr){2-4}\cmidrule(lr){5-7}\cmidrule(l){8-10}
 & \resultmetrics & \resultmetrics & \resultmetrics\\[-0.5pt]
\midrule[0.55pt]
\resultcategory{RT-1}
\rowcolor{resultstripe} SFT (50k) & 23.76 & 0.8298 & 0.1308 & 22.49 & 0.8075 & 0.1527 & 21.37 & 0.7860 & 0.1734\\
GRPO (200) & 23.97 & 0.8319 & 0.1289 & 22.60 & 0.8089 & 0.1517 & 21.39 & 0.7858 & 0.1739\\
\rowcolor{resultours} \textcolor{resultpurple}{\bfseries MemSPO (200)} & \textbf{24.33} & \textbf{0.8357} & \textbf{0.1240} & \textbf{23.42} & \textbf{0.8192} & \textbf{0.1397} & \textbf{22.45} & \textbf{0.8016} & \textbf{0.1565}\\
\addlinespace[0.5pt]
\textcolor{resultpurple}{\itshape Gain over GRPO} & \textcolor{resultpurple}{+1.49\%} & \textcolor{resultpurple}{+0.45\%} & \textcolor{resultpurple}{+3.80\%} & \textcolor{resultpurple}{+3.61\%} & \textcolor{resultpurple}{+1.28\%} & \textcolor{resultpurple}{+7.92\%} & \textcolor{resultpurple}{+4.97\%} & \textcolor{resultpurple}{+2.01\%} & \textcolor{resultpurple}{+9.98\%}\\[0pt]
\noalign{\vskip 2pt\hbox to\linewidth{\color{resultpending}\leaders\hbox{\rule{2pt}{0.35pt}\hskip 2pt}\hfill}\vskip 1pt}
\resultcategory{BridgeV2}
\rowcolor{resultstripe} SFT (50k) & 24.07 & 0.8300 & 0.1096 & 22.74 & 0.8084 & 0.1270 & 21.73 & 0.7940 & 0.1450\\
GRPO (200) & 24.12 & 0.8301 & 0.1089 & 22.80 & 0.8088 & 0.1263 & 21.80 & 0.7945 & 0.1441\\
\rowcolor{resultours} \textcolor{resultpurple}{\bfseries MemSPO (200)} & \textbf{24.31} & \textbf{0.8320} & \textbf{0.1073} & \textbf{23.00} & \textbf{0.8108} & \textbf{0.1244} & \textbf{21.97} & \textbf{0.7962} & \textbf{0.1418}\\
\addlinespace[0.5pt]
\textcolor{resultpurple}{\itshape Gain over GRPO} & \textcolor{resultpurple}{+0.83\%} & \textcolor{resultpurple}{+0.23\%} & \textcolor{resultpurple}{+1.44\%} & \textcolor{resultpurple}{+0.88\%} & \textcolor{resultpurple}{+0.25\%} & \textcolor{resultpurple}{+1.53\%} & \textcolor{resultpurple}{+0.78\%} & \textcolor{resultpurple}{+0.22\%} & \textcolor{resultpurple}{+1.58\%}\\[0pt]
\noalign{\vskip 2pt\hbox to\linewidth{\color{resultpending}\leaders\hbox{\rule{2pt}{0.35pt}\hskip 2pt}\hfill}\vskip 1pt}
\resultcategory{RoboCasa}
\rowcolor{resultstripe} SFT (50k) & 19.50 & 0.7957 & 0.1513 & 18.22 & 0.7689 & 0.1787 & 17.50 & 0.7519 & 0.1971\\
GRPO (200) & 19.68 & 0.7965 & 0.1514 & 18.64 & 0.7729 & 0.1758 & 18.03 & 0.7581 & 0.1919\\
\rowcolor{resultours} \textcolor{resultpurple}{\bfseries MemSPO (200)} & \textbf{19.92} & \textbf{0.7984} & \textbf{0.1476} & \textbf{18.88} & \textbf{0.7743} & \textbf{0.1718} & \textbf{18.29} & \textbf{0.7600} & \textbf{0.1877}\\
\addlinespace[0.5pt]
\textcolor{resultpurple}{\itshape Gain over GRPO} & \textcolor{resultpurple}{+1.21\%} & \textcolor{resultpurple}{+0.24\%} & \textcolor{resultpurple}{+2.47\%} & \textcolor{resultpurple}{+1.31\%} & \textcolor{resultpurple}{+0.19\%} & \textcolor{resultpurple}{+2.28\%} & \textcolor{resultpurple}{+1.45\%} & \textcolor{resultpurple}{+0.25\%} & \textcolor{resultpurple}{+2.17\%}\\[0pt]
\bottomrule[0.95pt]
\end{tabularx}
\vspace{0pt}
\end{table}
\endgroup

%% file: sections/rq4_results.tex
\needspace{8\baselineskip}
\subsection{Post-Training Effectiveness and Efficiency (RQ2)}
\label{sec:rq2-efficiency}

\begin{wrapfigure}[14]{r}{0.54\textwidth}
    \vspace{-6pt}
    \centering
    \setlength{\parskip}{0pt}
    \includegraphics[width=\linewidth]{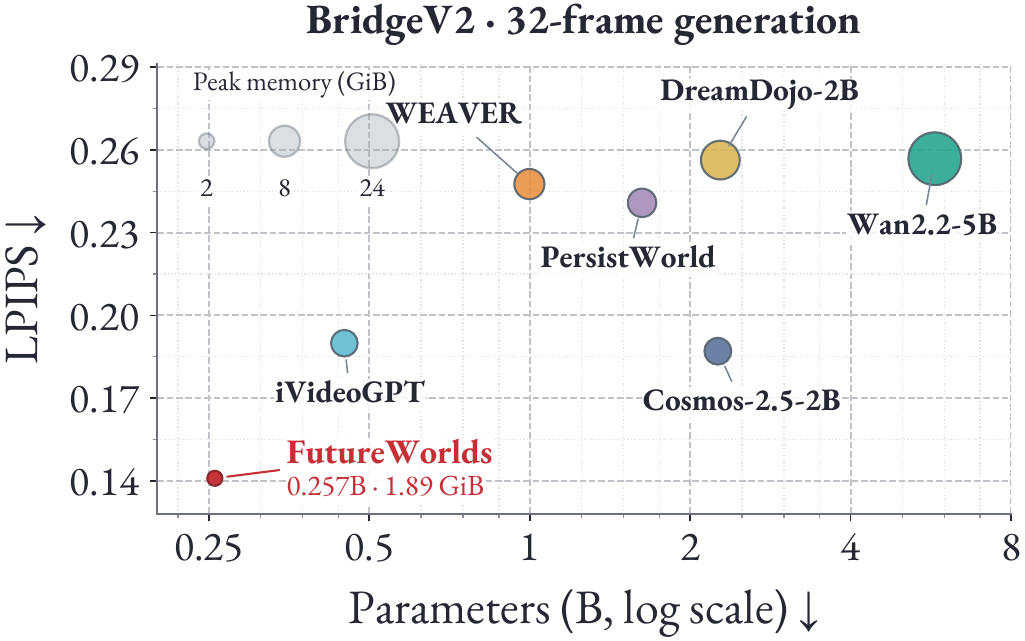}\par
    \setlength{\abovecaptionskip}{-2pt}
    \caption{\textbf{\textit{Model size and quality.}} Bubble area denotes peak GPU memory on one H20.}
    \label{fig:efficiency-bubble}
    \vspace{-6pt}
\end{wrapfigure}

On 128 matched trajectories per dataset with decoding and memory fixed, MemSPO leads all three metrics at 10, 20, and 32 frames after 200 RL updates (\mbox{Table~\ref{tab:matched-post-training}}). At 32 frames, LPIPS drops by 9.98\%, 1.58\%, and 2.17\% versus GRPO on RT-1, BridgeV2, and RoboCasa, showing gains beyond ten-frame training.

On 32 monitoring clips, 200 updates reduce LPIPS by 4.02\%/3.45\% on RT-1/BridgeV2 versus SFT (\mbox{Figure~\ref{fig:rq4-learning}a,b}), with improved object consistency and motion (c). Gains persist across five decoding settings (\mbox{Appendix Table~\ref{tab:decoding-sensitivity}}).

\par\WFclear
With checkpoints and beam width fixed, the LPIPS reduction remains 3.03--4.32\% on RT-1 and 2.60--3.84\% on BridgeV2 across the tested temperature and top-$p$ settings. A matched ordinary-beam post-training comparison also favors MemSPO in PSNR and LPIPS across all three datasets and horizons under identical evaluation decoding and memory settings (Appendix~\ref{app:ordinary-beam-posttraining}, Table~\ref{tab:ordinary-beam-posttraining}).

\noindent\textbf{Resource footprint.}
\method{} leads seven models in BridgeV2 LPIPS with 0.257B inference parameters (visual codec included) and 1.89\,GiB peak memory (Figure~\ref{fig:efficiency-bubble}). Costs use three 32-frame rollouts at batch size one, with native configurations and cached language features. Beam decoding takes 26.76\,s versus iVideoGPT's 4.14\,s per rollout. Smaller models do not ensure faster inference; Appendix~\ref{app:resource-measurements} details all measurements and settings.

%% file: sections/motion_results.tex
\begin{figure}[!t]
    \centering
    \includegraphics[width=\linewidth]{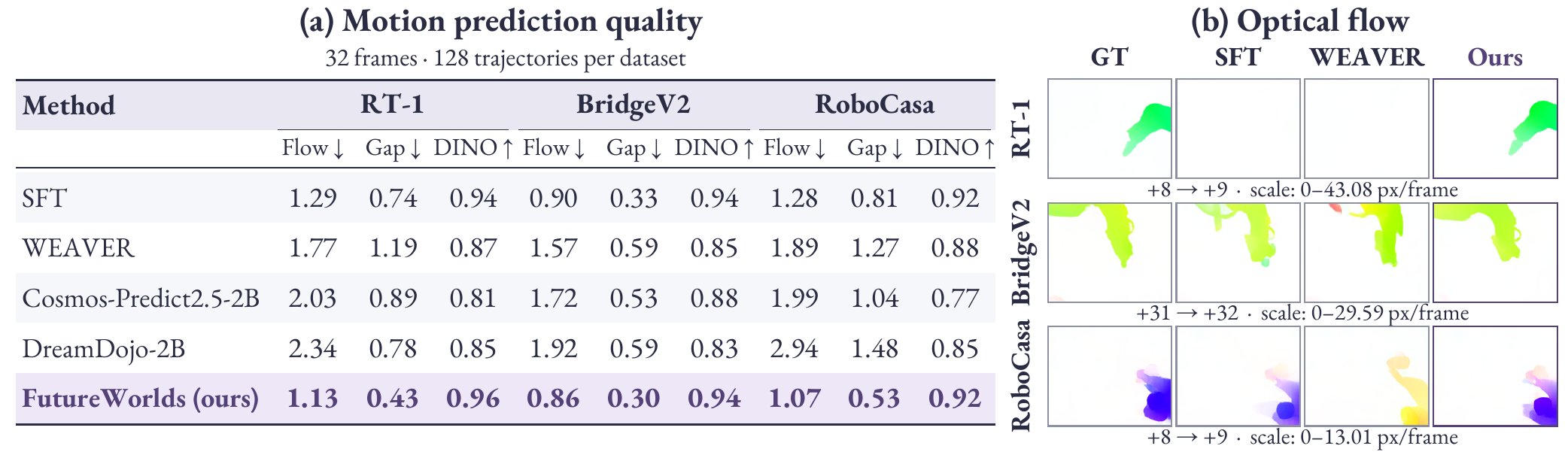}\par
    \setlength{\abovecaptionskip}{2pt}
    \caption{\textbf{\textit{Motion quality and optical-flow comparisons.}}
    (a) Full-cohort results at 32 frames. (b) Selected examples with a shared flow-color scale within each dataset.}
    \label{fig:motion-quality}
\end{figure}

\vspace{-3pt}
\subsection{Motion and Dynamic Prediction (RQ3)}
\label{sec:motion-quality}

Using WorldArena's flow and feature components~\citep{shang2026worldarena}, we evaluate 128 trajectories per dataset against SFT, WEAVER, Cosmos, and DreamDojo (Figure~\ref{fig:motion-quality}a). Flow error measures vector disagreement, magnitude gap measures motion-strength mismatch, and DINO similarity measures frame-level feature agreement (Appendix~\ref{app:motion-quality}). \method{} leads all three diagnostics, reducing flow error by 12.40\%, 4.57\%, and 16.68\% versus SFT on RT-1, BridgeV2, and RoboCasa. Improvements in both flow error and magnitude gap support closer motion agreement rather than merely larger movements. The shared-scale examples (b) complement these cohort-level results.

%% file: sections/memory_results.tex
\vspace{-3pt}
\subsection{Memory Maintenance (RQ4)}
\label{sec:memory-ablation}

We fix weights and decoding and vary only the retained history: the latest state, six recent states, or an initial anchor plus six recent states. Full memory leads PSNR, SSIM, and MSE at 32 frames across all datasets (Figure~\ref{fig:memory-ablation}), reducing MSE by 40.83\%, 57.69\%, and 29.22\% versus removing historical memory on RT-1, BridgeV2, and RoboCasa. Recent history supplies most gains; the anchor adds smaller benefits. As no variant is retrained, this comparison isolates the contribution of retained context at inference. Appendix~\ref{app:memory-ablation} reports settings and results across horizons.

\begingroup
\setlength{\intextsep}{3pt}
\begin{figure}[H]
\centering
    \includegraphics[width=\linewidth]{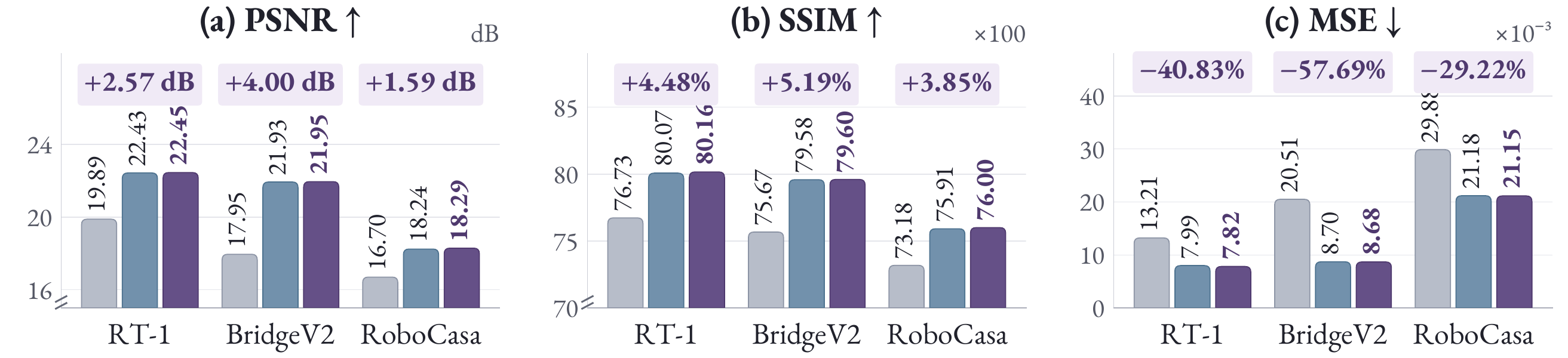}\par
    \vspace{-2pt}
    \setlength{\abovecaptionskip}{2pt}
    \caption{\textbf{\textit{Memory ablation at 32 frames.}}
    \mbox{\textcolor[HTML]{B7BDC9}{\rule{0.65em}{0.65em}}\,w/o Historical Memory},
    \mbox{\textcolor[HTML]{7291AC}{\rule{0.65em}{0.65em}}\,w/o Initial Anchor}, and
    \mbox{\textcolor[HTML]{655087}{\rule{0.65em}{0.65em}}\,Full Model}.
    Gains are versus w/o Historical Memory; PSNR and SSIM axes are truncated.}
    \label{fig:memory-ablation}
\end{figure}
\endgroup
\vspace{-3pt}

%% file: sections/conclusion.tex
\section{Conclusion and Future Work}
\label{sec:conclusion}

We present \method{}, which learns from alternative futures through coordinated candidate construction, history maintenance, and quality comparison. \memspo{} combines diverse search with history-consistent optimization, improving generation within hundreds of updates and beyond the training horizon. Three-dataset evaluations support gains in visual quality, object consistency, and motion prediction. BridgeV2 comparisons also show better quality with fewer parameters and lower measured memory usage, although autoregressive beam search incurs latency. Future work focuses on faster decoding and validating these benefits in closed-loop planning and robot control.

%% file: sections/statements.tex
\section*{AI Use Statement}
We use generative AI tools to assist with manuscript drafting, translation, and language editing; implementation and debugging of research and evaluation code; refinement of experimental protocols; qualitative case inspection and interpretation of results; review of mathematical derivations; and preparation of figures and tables. Reported numerical results come from experimental outputs rather than language-model-generated estimates. The authors take full responsibility for verifying AI-assisted work and for the accuracy and integrity of the final text, references, claims, code, and other artifacts.

\section*{Reproducibility Statement}
Section~\ref{sec:method} presents the model and optimization procedure. Appendix~\ref{app:implementation-details} documents initialization, supervised fine-tuning, post-training, reward designs, baseline adaptation, and evaluation protocols. Appendices~\ref{app:ordinary-beam-posttraining}, \ref{app:memory-ablation}, and~\ref{app:resource-measurements} provide the settings and full results for candidate-construction, memory, and resource comparisons. The anonymous project link in the abstract provides the code and reproduction instructions; we plan to release the trained checkpoints.

%% file: sections/appendix.tex
\numberwithin{equation}{section}
\section{Theoretical Analysis}
\label{app:appendix}
\label{app:theory}

We analyze how alternative futures supply a learning signal in \method{}. On a fixed candidate set, the local reward gradient of \memspo{} admits a pairwise decomposition that links relative rewards to differences in policy score gradients. The analysis characterizes this update; it does not establish that search produces better candidates.

\subsection{Pairwise Structure of the Reward Gradient}
\label{app:pairwise-gradient}

For a fixed input, consider $G\ge2$ candidates with $L=HN$ tokens each. Flatten the token indices into $k\in\{1,\ldots,L\}$, with $z_k^{(i)}$ and $u_k^{(i)}$ denoting candidate $i$'s token and condition, including its memory and within-frame prefix. Tokens, conditions, rewards, and normalization statistics remain fixed during differentiation. Selected tokens have positive old-policy probabilities, and current-policy probabilities are differentiable near $\theta_{\mathrm{old}}$.

Let $\sigma_R=\operatorname{Std}(R_1,\ldots,R_G)$ use the same convention as Eq.~(\ref{eq:relative_advantage}). Define the group-relative advantage and mean token score gradient as
\begin{equation}
 A_i=\frac{R_i-\bar R}{\sigma_R+\varepsilon},\qquad
 \bar R=\frac1G\sum_{i=1}^{G}R_i,\qquad
 \mathbf s_i=\left.\frac1L\sum_{k=1}^{L}\nabla_\theta\log\pi_\theta(z_k^{(i)}\mid u_k^{(i)})\right|_{\theta=\theta_{\mathrm{old}}}.
 \label{eq:app-score-definitions}
\end{equation}
Here, $\varepsilon>0$ stabilizes reward normalization. The implementation uses sample standard deviation, so $\sigma_R^2=\frac{1}{G-1}\sum_i(R_i-\bar R)^2$. We first analyze the normalized advantages above; the numerical zero-variance guard is addressed below. The reward component of Eq.~(\ref{eq:memspo_objective}) is
\begin{equation}
 \mathcal J_R(\theta)=\frac1{GL}\sum_{i=1}^{G}\sum_{k=1}^{L}
 \ell_{\mathrm{clip}}\!\left(\rho_k^{(i)}(\theta),A_i\right),\qquad
 \rho_k^{(i)}(\theta)=\frac{\pi_\theta(z_k^{(i)}\mid u_k^{(i)})}{\pi_{\mathrm{old}}(z_k^{(i)}\mid u_k^{(i)})}.
 \label{eq:app-reward-objective}
\end{equation}

\begin{proposition}[Pairwise reward-gradient identity]
\label{prop:pairwise-update}
Suppose generation and policy scoring use matching candidate histories, the clipping radius satisfies $0<\epsilon_{\mathrm c}<1$, and the negative-advantage dual-clipping constant satisfies $c_{\mathrm d}>1$. At $\theta=\theta_{\mathrm{old}}$, the reward gradient admits the pairwise representation
\begin{equation}
 \mathbf g_R:=\left.\nabla_\theta\mathcal J_R(\theta)\right|_{\theta=\theta_{\mathrm{old}}}
 =\frac1G\sum_{i=1}^{G}A_i\mathbf s_i
 =\frac{\sum_{i<j}(R_i-R_j)(\mathbf s_i-\mathbf s_j)}{G^2(\sigma_R+\varepsilon)}.
 \label{eq:app-pairwise-identity}
\end{equation}
\end{proposition}

\begin{proof}
We first reduce the clipped objective locally. Write
\begin{equation}
 v(\rho,A)=\min\!\left\{\rho A,\operatorname{clip}(\rho,1-\epsilon_{\mathrm c},1+\epsilon_{\mathrm c})A\right\}.
 \label{eq:app-clip-definition}
\end{equation}
The policy term is $v(\rho,A)$ for $A\ge0$ and $\max\{v(\rho,A),c_{\mathrm d}A\}$ for $A<0$. At identical parameters, evaluating numerator and denominator under the same fixed condition gives $\rho_k^{(i)}(\theta_{\mathrm{old}})=1$. Matching this condition to generation additionally ensures that the score refers to the history that produces the candidate. In a neighborhood of this point, ordinary clipping is inactive. For $A_i<0$, the strict inequality $A_i>c_{\mathrm d}A_i$ also makes dual clipping inactive; for $A_i=0$, the term and its derivative are zero. Consequently,
\begin{equation}
 \left.\nabla_\theta\ell_{\mathrm{clip}}(\rho_k^{(i)}(\theta),A_i)\right|_{\theta=\theta_{\mathrm{old}}}
 =A_i\left.\nabla_\theta\rho_k^{(i)}(\theta)\right|_{\theta=\theta_{\mathrm{old}}}.
 \label{eq:app-local-clip-gradient}
\end{equation}
The denominator of the probability ratio is fixed, so its derivative satisfies
\begin{equation}
 \nabla_\theta\rho_k^{(i)}(\theta)
 =\frac{\nabla_\theta\pi_\theta(z_k^{(i)}\mid u_k^{(i)})}{\pi_{\mathrm{old}}(z_k^{(i)}\mid u_k^{(i)})}
 =\rho_k^{(i)}(\theta)\nabla_\theta\log\pi_\theta(z_k^{(i)}\mid u_k^{(i)}).
 \label{eq:app-ratio-derivative}
\end{equation}
Substituting into Eq.~(\ref{eq:app-reward-objective}) and evaluating at the old policy yields
\begin{equation}
 \mathbf g_R=\frac1G\sum_{i=1}^{G}A_i\mathbf s_i
 =\frac1{G(\sigma_R+\varepsilon)}\sum_{i=1}^{G}(R_i-\bar R)\mathbf s_i.
 \label{eq:app-centered-gradient}
\end{equation}

To obtain the pairwise form, symmetry and the vanishing diagonal terms give
\begin{equation}
 \sum_{i<j}(R_i-R_j)(\mathbf s_i-\mathbf s_j)
 =\frac12\sum_{i=1}^{G}\sum_{j=1}^{G}
 \left(R_i\mathbf s_i+R_j\mathbf s_j-R_i\mathbf s_j-R_j\mathbf s_i\right).
 \label{eq:app-pairwise-expansion}
\end{equation}
Collecting the four sums and using $\sum_iR_i=G\bar R$ gives
\begin{equation}
 \sum_{i<j}(R_i-R_j)(\mathbf s_i-\mathbf s_j)
 =G\sum_{i=1}^{G}R_i\mathbf s_i
 -\left(\sum_{i=1}^{G}R_i\right)\left(\sum_{i=1}^{G}\mathbf s_i\right)
 =G\sum_{i=1}^{G}(R_i-\bar R)\mathbf s_i.
 \label{eq:app-pairwise-centering}
\end{equation}
Combining Eqs.~(\ref{eq:app-centered-gradient}) and~(\ref{eq:app-pairwise-centering}) proves Eq.~(\ref{eq:app-pairwise-identity}).
\end{proof}

Each pair contributes through its reward and score-gradient differences; either vanishing difference removes its contribution. No independence assumption is required, so the identity also holds for beam-search candidates. The algebra applies to any fixed conditions shared by numerator and denominator; history matching gives these scores their intended candidate-generation interpretation, rather than providing an independent performance guarantee.

\subsection{A Bound under Candidate Collapse}
\label{app:candidate-collapse}

Define the empirical reward variance and score-gradient dispersion by
\begin{equation}
 \bar{\mathbf s}=\frac1G\sum_{i=1}^{G}\mathbf s_i,\qquad
 V_R=\frac1G\sum_{i=1}^{G}(R_i-\bar R)^2,\qquad
 V_s=\frac1G\sum_{i=1}^{G}\|\mathbf s_i-\bar{\mathbf s}\|_2^2.
 \label{eq:app-dispersion-definitions}
\end{equation}
We distinguish $V_R$ from $\sigma_R^2$ so that the result accommodates either population or sample standard deviation in advantage normalization.

\begin{corollary}[Reward-gradient magnitude]
\label{cor:candidate-collapse}
Under the conditions of Proposition~\ref{prop:pairwise-update},
\begin{equation}
 \|\mathbf g_R\|_2\le\frac{\sqrt{V_RV_s}}{\sigma_R+\varepsilon}.
 \label{eq:app-gradient-bound}
\end{equation}
In particular, the reward gradient is zero if all candidate rewards are equal or all candidate score gradients coincide.
\end{corollary}

\begin{proof}
Since $\sum_i(R_i-\bar R)=0$, subtracting the mean score gradient leaves Eq.~(\ref{eq:app-centered-gradient}) unchanged:
\begin{equation}
 \mathbf g_R=\frac1{G(\sigma_R+\varepsilon)}
 \sum_{i=1}^{G}(R_i-\bar R)(\mathbf s_i-\bar{\mathbf s}).
 \label{eq:app-double-centering}
\end{equation}
The triangle inequality gives
\begin{equation}
 \|\mathbf g_R\|_2\le\frac1{G(\sigma_R+\varepsilon)}
 \sum_{i=1}^{G}|R_i-\bar R|\,\|\mathbf s_i-\bar{\mathbf s}\|_2.
 \label{eq:app-triangle-bound}
\end{equation}
Applying the Cauchy--Schwarz inequality to the two scalar sequences yields
\begin{equation}
 \|\mathbf g_R\|_2\le
 \frac{\sqrt{\sum_i(R_i-\bar R)^2}\sqrt{\sum_i\|\mathbf s_i-\bar{\mathbf s}\|_2^2}}{G(\sigma_R+\varepsilon)}
 =\frac{\sqrt{V_RV_s}}{\sigma_R+\varepsilon}.
 \label{eq:app-cauchy-bound}
\end{equation}
If either $V_R$ or $V_s$ vanishes, the upper bound is zero.
\end{proof}

We can also express dispersion through pairwise distances. Expanding the squared norms gives
\begin{equation}
 \sum_{i<j}\|\mathbf s_i-\mathbf s_j\|_2^2
 =G\sum_{i=1}^{G}\|\mathbf s_i\|_2^2-\left\|\sum_{i=1}^{G}\mathbf s_i\right\|_2^2
 =G\sum_{i=1}^{G}\|\mathbf s_i-\bar{\mathbf s}\|_2^2.
 \label{eq:app-pairwise-dispersion}
\end{equation}
Hence $V_s=G^{-2}\sum_{i<j}\|\mathbf s_i-\mathbf s_j\|_2^2$. If every pair satisfies $\|\mathbf s_i-\mathbf s_j\|_2\le\delta$, the $G(G-1)/2$ pairs imply
\begin{equation}
 V_s\le\frac{G-1}{2G}\delta^2,\qquad
 \|\mathbf g_R\|_2\le
 \frac{\sqrt{V_R}}{\sigma_R+\varepsilon}\sqrt{\frac{G-1}{2G}}\,\delta.
 \label{eq:app-collapse-bound}
\end{equation}
Thus, at a fixed reward scale, collapse in candidate score gradients limits the reward-driven update. Larger dispersion alone does not guarantee a larger or better-directed update, since pairwise contributions can cancel. Visual diversity is also distinct from score-gradient dispersion; connecting the two requires empirical analysis.

\noindent\textbf{\emph{Scope.}}
These statements concern the reward component at $\theta_{\mathrm{old}}$; the reference-regularization gradient can remain nonzero under collapse, and clipping can become active away from this point. Because search changes the candidate distribution, the analysis implies neither an unbiased on-policy gradient nor monotonic improvement. The benefit of diverse beam search remains an empirical question. In our implementation, each fresh candidate set receives one update evaluated at the old parameters, where clipping is locally inactive. If $\sigma_R\le10^{-6}$, the implementation sets all advantages to zero; the reward gradient is then zero and the magnitude bound remains valid, although the unguarded pairwise equality need not hold for small nonzero reward differences. The full optimizer step also includes reference regularization and optimizer preconditioning, so it is distinct from $\mathbf g_R$.

\section{Implementation Details}
\label{app:implementation-details}

\subsection{Pretrained Initialization, SFT, and RL}
\label{app:tokenizer-details}
\label{app:training-settings}

\noindent\textbf{Pretrained initialization and tokenizer adaptation.}
We initialize from released pretrained visual-tokenizer and autoregressive world-model checkpoints. RT-1 reuses the pretrained tokenizer directly; BridgeV2 and RoboCasa use dataset-adapted tokenizers. Adaptation optimizes the reconstruction objective in Eq.~\ref{eq:tokenizer_objective}, with unit weights on $\mathcal L_1$ and VGG-LPIPS and an adversarial weight of $0.1$. We use a global batch size of 32, a learning rate of $5\times10^{-5}$, and 5,000 warmup steps. The BridgeV2 and RoboCasa tokenizer checkpoints use 60,000 and 50,000 shared generator/discriminator steps, respectively; these counters do not denote generator-only updates. The resulting tokenizers remain frozen during world-model SFT and RL.

\noindent\textbf{Supervised fine-tuning.}
Each dataset receives 50,000 SFT updates with global batch size 64. Teacher forcing supplies recorded histories, and cross-entropy supervises target dynamic tokens only. AdamW uses a backbone learning rate of $10^{-5}$, a text-adapter learning rate of $10^{-4}$ where applicable, zero weight decay, and 500 warmup updates. RT-1 and RoboCasa use cosine decay to one tenth of the peak learning rate; BridgeV2 keeps the learning rate constant after warmup. Gradient norms are clipped at 1.0. We update the autoregressive backbone, action embeddings, and text adapters while freezing the visual tokenizer and T5 encoder. RT-1 and RoboCasa use language conditioning; BridgeV2 omits it. The SFT checkpoint initializes both the trainable RL policy and its frozen reference.

\noindent\textbf{Reinforcement-learning post-training.}
MemSPO uses ten-frame predicted rollouts, global batch size 16, and fresh AdamW state with learning rate $10^{-6}$ and zero weight decay. Four search groups contain two beams each, with diversity penalty $0.1$; one trajectory per group enters reward comparison. The clipping threshold is $0.2$, the reference-regularization coefficient is $0.001$, and the gradient-norm limit is 1.0. Each fresh candidate set receives one optimizer update. Candidate generation and policy scoring retain the same scene/state anchors and six recent states; future ground-truth frames enter reward computation only. Main-table checkpoints use 500/200/400 updates for RT-1/BridgeV2/RoboCasa, whereas the matched GRPO--MemSPO comparison uses 200 updates on every dataset. Dataset-specific reward choices are identified in Appendix~\ref{app:reward-design}; Algorithm~\ref{alg:futureworlds} details the update.

\clearpage
\subsection{Training Algorithm}
\label{app:training-algorithm}

Algorithm~\ref{alg:futureworlds} summarizes one-input updates; distributed training averages the losses across inputs. The colors identify candidate construction, history maintenance, and relative-quality learning. The tokenizer, text encoder, and SFT reference policy remain frozen throughout post-training.

\begingroup
\definecolor{algconstruct}{HTML}{405F7D}
\definecolor{algmaintain}{HTML}{426E62}
\definecolor{alglearn}{HTML}{705187}
\fontencoding{T1}\fontfamily{EBGaramond-TLF}\selectfont
\fontsize{10}{11.8}\selectfont
\begin{tcolorbox}[enhanced,colback=white,colframe=algconstruct!55,
 colbacktitle=algconstruct!9,coltitle=algconstruct,boxrule=0.6pt,arc=3pt,
 left=8pt,right=8pt,top=6pt,bottom=6pt,
 title={\textbf{Algorithm 1}\quad FutureWorlds: learning from alternative futures},
 fonttitle=\fontencoding{T1}\fontfamily{EBGaramond-TLF}\selectfont\bfseries]
\newcounter{fwalgorithm}\refstepcounter{fwalgorithm}\label{alg:futureworlds}
\textbf{Input:} SFT parameters $\theta_0$; training examples $x=(\mathbf c_v,\mathbf a,\ell,o_{1:H})$; horizon $H$; $N$ tokens/frame; $G$ groups, $b$ branches/group; memory budget $B$.\par
\textbf{Output:} Updated world-model parameters $\theta$.\par\smallskip
\begin{enumerate}[leftmargin=1.8em,label=\textcolor{algconstruct!65}{\scriptsize\arabic*},labelsep=0.6em,itemsep=1pt,parsep=0pt,topsep=2pt]
\item Initialize $\theta\leftarrow\theta_0$ and freeze $\pi_{\mathrm{ref}}\leftarrow\pi_{\theta_0}$.
\item \textbf{For} each post-training iteration \textbf{do}
\item \quad Draw $x$; set $\theta_{\mathrm{old}}\leftarrow\operatorname{stopgrad}(\theta)$.
\item \quad Initialize $G$ groups of $b$ branches with the same scene and state anchors.
\item \quad Set the first branch score in each group to $0$ and the others to $-\infty$.
\item \quad \textcolor{algconstruct}{\textbf{Construct}} / \textcolor{algmaintain}{\textbf{Maintain}} \textbf{for} $t=1,\ldots,H$ \textbf{do}
\item \qquad Build each branch's prefix from its memory and actions; reset its KV cache.
\item \qquad \textbf{For} token position $j=1,\ldots,N$ \textbf{do}
\item \qquad\quad \textbf{For} group $g=1,\ldots,G$ \textbf{do}
\item \qquad\qquad Compute valid-visual-token log probabilities with $\pi_{\mathrm{old}}$.
\item \qquad\qquad Extend branches using $S_g+\log\pi_{\mathrm{old}}-\lambda_{\mathrm{div}}d_g$ (Eq.~\ref{eq:diverse_search}).
\item \qquad\qquad Retain the top $b$ extensions; reorder prefixes, caches, and ancestry.
\item \qquad\quad \textbf{End for}; \textbf{end for}
\item \qquad Reorder prior trajectory records and memories by the selected ancestry.
\item \qquad Append each branch's frame; retain anchors and its recent history within $B$.
\item \quad \textbf{End for}; return the highest-search-score trajectory from each group.
\item \quad \textcolor{alglearn}{\textbf{Compare}}: decode the $G$ candidates and compute $R_i$ (Appendix~\ref{app:reward-design}).
\item \quad Compute $A_i=(R_i-\bar R)/(\operatorname{Std}(R)+\varepsilon)$; if $\operatorname{Std}(R)\le10^{-6}$, set $A_i=0$.
\item \quad \textcolor{alglearn}{\textbf{Learn}}: score candidate tokens under their saved generation-time histories.
\item \quad Use these same conditions for old, current, and frozen reference policies.
\item \quad Form $\mathcal J_{\mathrm{MemSPO}}$ (Eq.~\ref{eq:memspo_objective}); take one optimizer step on $-\mathcal J$.
\item \textbf{End for}; \textbf{return} $\theta$.
\end{enumerate}
\smallskip
\textcolor{algconstruct}{\textit{Search penalties select candidates; unmodified policy probabilities score them.}}
\end{tcolorbox}
\endgroup

\noindent\textbf{\emph{Implementation correspondence.}}
Search scores accumulate across the full horizon. Candidate selection also reorders previously saved frame prefixes and responses, preserving each final trajectory's ancestry. Reward computation and candidate construction are detached from differentiation; ground-truth futures enter rewards only. The reference penalty uses the numerically clipped form $\operatorname{clip}(q-\log q-1,-10,10)$. This training search differs from evaluation-time Beam4, which selects one frame prediction before proceeding to the next frame.

\clearpage
\subsection{Flexible Reward Design}
\label{app:reward-design}

MemSPO separates trajectory evaluation from parameter optimization. A reward module assigns a scalar score to each candidate future, and group normalization converts these scores into the advantages in Eq.~\ref{eq:relative_advantage}. Replacing the reward therefore leaves candidate construction, history maintenance, and the update objective unchanged. We describe two implemented designs that emphasize complementary aspects of prediction quality.

\noindent\textbf{\emph{Global reconstruction reward.}}
The base design in Eq.~\ref{eq:trajectory_reward} measures whole-frame pixel and perceptual agreement. For candidate $i$, define
\begin{equation}
 L_{\mathrm{global},t}^{(i)}
 =\operatorname{MAE}(\hat o_t^{(i)},o_t)
 +\operatorname{LPIPS}_{\mathrm{VGG}}(\hat o_t^{(i)},o_t),
 \qquad
 R_{\mathrm{base}}^{(i)}=-\frac{1}{H}\sum_{t=1}^{H}L_{\mathrm{global},t}^{(i)}.
\end{equation}
MAE uses RGB values in $[0,1]$, while frozen VGG-LPIPS receives images mapped to $[-1,1]$. This design evaluates overall appearance without requiring region annotations.

\noindent\textbf{\emph{Motion-aware extension.}}
To emphasize changing regions and temporal variation, the extended design adds local appearance and frame-difference errors:
\begin{equation}
 R_{\mathrm{motion}}^{(i)}=-\frac{1}{H}\sum_{t=1}^{H}
 \left[L_{\mathrm{global},t}^{(i)}
 +\lambda_{\mathrm{roi}}L_{\mathrm{roi},t}^{(i)}
 +\lambda_{\Delta}L_{\Delta,t}^{(i)}\right],
 \label{eq:motion-reward}
\end{equation}
\begin{equation}
 L_{\Delta,t}^{(i)}=
 \operatorname{MAE}\!\left(
 \hat o_t^{(i)}-\hat o_{t-1}^{(i)},\;o_t-o_{t-1}\right).
\end{equation}
At the first predicted frame, both preceding images equal the last observed frame. The local term $L_{\mathrm{roi},t}^{(i)}$ combines masked MAE with LPIPS on the corresponding bounding crop. All candidates share a mask computed from the channel-averaged absolute ground-truth frame difference. We threshold it above the larger of its 90th percentile and $0.02$, dilate it with an $11\times11$ window, and use the full frame if the mask covers less than $1\%$ of the image. Masked MAE is normalized by the number of selected RGB values; the LPIPS crop spans at least $64\times64$ pixels. This region emphasizes visible changes without semantic segmentation or optical-flow estimation.

\noindent\textbf{\emph{Fixed calibration.}}
We set $\lambda_{\mathrm{roi}}=0.5s_{\mathrm{global}}/s_{\mathrm{roi}}$ and $\lambda_{\Delta}=0.25s_{\mathrm{global}}/s_{\Delta}$. Each scale is the median temporally averaged component loss over candidates from 64 training-only calibration episodes using the frozen SFT checkpoint. Four independently sampled and four diverse-search candidates per input yield 512 calibration trajectories. The scales are $s_{\mathrm{global}}=0.153368$, $s_{\mathrm{roi}}=0.190531$, and $s_{\Delta}=0.026607$; they remain fixed throughout post-training.

\noindent\textbf{\emph{Experimental configurations.}}
Main-table RT-1 and RoboCasa checkpoints use $R_{\mathrm{base}}$, whereas BridgeV2 uses $R_{\mathrm{motion}}$. The matched GRPO--MemSPO comparison and the ordinary-versus-diverse search training comparison use $R_{\mathrm{base}}$ throughout, holding the reward fixed within each comparison. Both designs use ten-frame training rollouts and the same group-relative update. Ground-truth futures enter reward computation only; they never condition candidate generation or populate candidate memory. These designs demonstrate a shared optimization interface for different quality objectives, rather than establishing that every reward choice yields equivalent performance.

\clearpage
\subsection{Baseline Adaptation and Evaluation}
\label{app:baseline-adaptation}
\label{app:main-table-protocols}

\noindent\textbf{Shared evaluation.}
Table~\ref{tab:main-results} evaluates 32-frame predictions on 128 matched trajectories per dataset. RT-1 and BridgeV2 use held-out development sets; RoboCasa uses a held-out test set. Predictions are resized to the reference resolution, and LPIPS-VGG, SSIM, and frame-wise PSNR are averaged over frames and trajectories. Recorded actions condition generation; ground-truth future images are evaluation targets and never refresh rollout histories. Models retain their native input interfaces and decoding procedures. Cosmos-Predict2.5-2B uses one initial observation, whereas \method{} uses two. Ten-frame baseline results use prefixes of the same 32-frame rollouts.

\noindent\textbf{FutureWorlds checkpoints.}
Main-table evaluations use a beam width of four and persistent memory, with 500/200/400 MemSPO updates for RT-1/BridgeV2/RoboCasa. BridgeV2 uses the motion-aware reward (Appendix~\ref{app:reward-design}) and omits text conditioning. The matched post-training comparison instead uses 200 updates for both RL methods on every dataset.

\noindent\textbf{iVideoGPT and PersistWorld.}
On BridgeV2, iVideoGPT adapts the official OXE-256 tokenizer and autoregressive backbone for 20,000 updates each, freezing the tokenizer during backbone training. It uses two observations, recorded seven-dimensional actions, temperature 1, and top-$k$ 100 sampling. Ten-frame chunks condition subsequent chunks on the two latest generated frames. PersistWorld starts from the released Ctrl-World checkpoint, performs 2,000 target-domain updates followed by 30 NFT updates, and uses generated latent history with 25 denoising steps and guidance scale 2.0. Its visual tokenizer and CLIP encoders remain frozen during base adaptation. Neither rollout accesses future ground-truth frames.

\noindent\textbf{Wan2.2 adaptation.}
\label{app:wan-adaptation}
We adapt the official Wan2.2-TI2V-5B checkpoint~\citep{wan2025wan} using action and history adapters plus rank-32 LoRA~\citep{hu2021lora} in attention layers (Figure~\ref{fig:wan-action-adaptation}). Following chunk-wise conditioning~\citep{gao2026dreamdojo}, four actions modulate each future latent step; the previous observation supplies spatial history. The latest observation remains a clean anchor. Only adapters and LoRA parameters are trained with future-latent flow matching; the pretrained transformer, VAE, and text encoder remain frozen. The BridgeV2 run uses two observed and twelve future frames, training-only action normalization, and AdamW with a cosine learning-rate schedule from $10^{-4}$ to $10^{-5}$. A disjoint 32-example monitor selects the 8,000-update checkpoint using LPIPS-Alex; the final 128-example evaluation uses LPIPS-VGG. Inference uses 50 UniPC steps~\citep{zhao2023unipc}, with guidance scale and time shift both set to 5.

\begin{figure}[H]
  \centering
  \includegraphics[width=\linewidth]{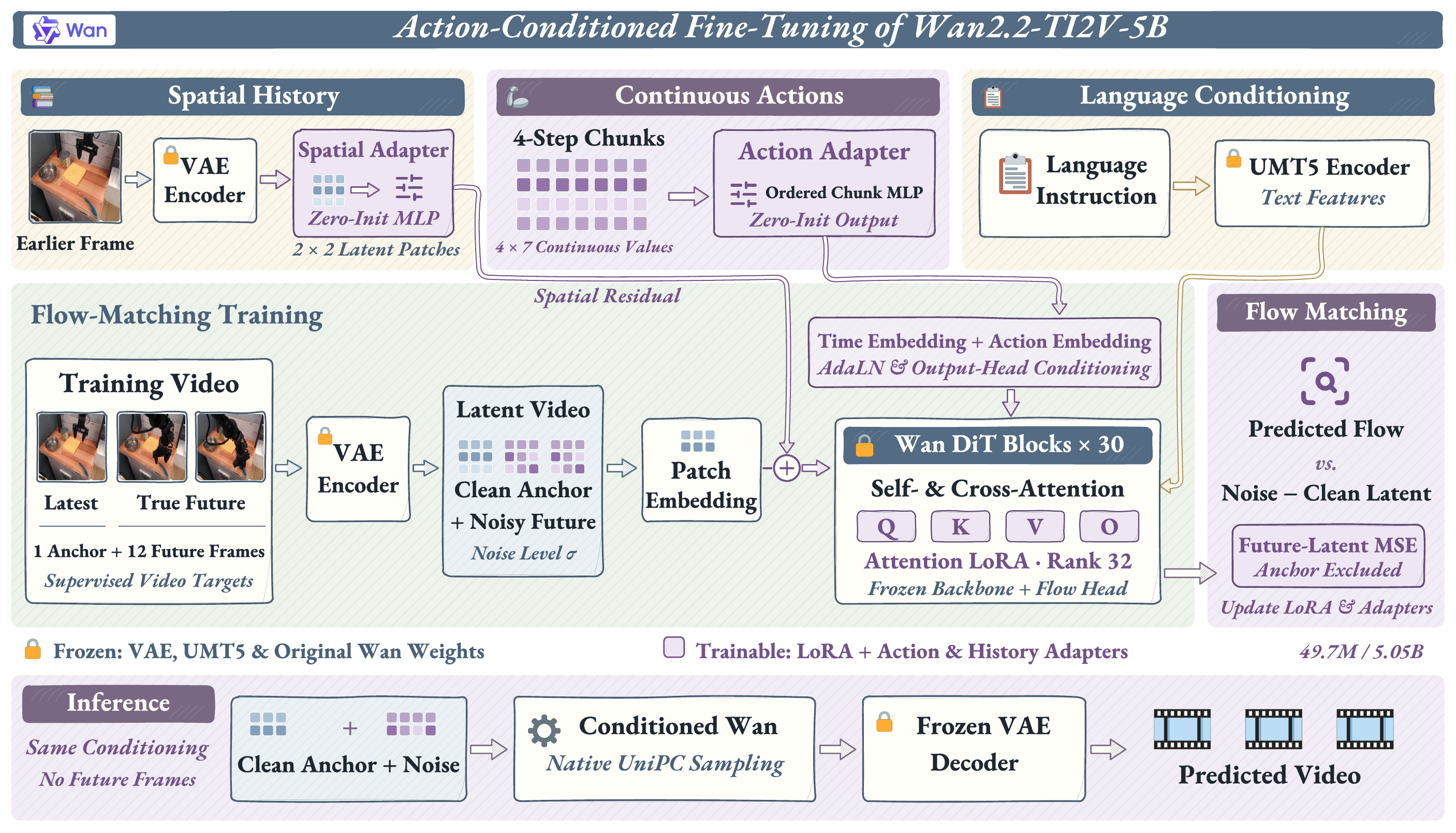}
  \setlength{\abovecaptionskip}{3pt}
  \caption{\textbf{Action-conditioned Wan2.2 adaptation.} Action chunks modulate time embeddings, and a history adapter supplies spatial context. Adapters and LoRA parameters are trainable; pretrained weights remain frozen.}
  \label{fig:wan-action-adaptation}
\end{figure}

\noindent\textbf{Evaluation metrics.}
PSNR measures pixel fidelity as $-10\log_{10}(\mathrm{MSE})$ in decibels for RGB images in $[0,1]$. SSIM compares local luminance, contrast, and structure. LPIPS~\citep{zhang2018lpips} measures perceptual distance with pretrained VGG features and inputs in $[-1,1]$. Higher PSNR/SSIM and lower LPIPS are better. Frame-level scores, including PSNR in decibels, are averaged over the prediction horizon and then over trajectories.

\noindent\textbf{Release plan.}
We plan to release all implementation code and trained checkpoints, including adaptation configurations and evaluation scripts.

\Needspace{7\baselineskip}
\subsection{Learning-Curve and Qualitative Evaluation}
\label{app:rq4-protocol}

Figure~\ref{fig:rq4-learning}(a,b) evaluates 32 fixed monitoring clips per dataset with 10-frame prediction and a beam width of four. All checkpoints share inputs and decoding settings. Circles show measurements, curves are endpoint-preserving smoothed guides, and stars mark 200-update MemSPO checkpoints. SFT and RL use different horizontal scales; update counts do not imply equal compute.

Figure~\ref{fig:rq4-learning}(c) shows selected improved examples at frame 32. The RT-1, BridgeV2, and RoboCasa models use 500, 200, and 400 MemSPO updates, respectively; the BridgeV2 example uses the base-reward checkpoint underlying the learning curve. Each SFT comparison matches inputs and decoding. Full frames remain visible, and detail crops use identical coordinates across methods. These examples illustrate prediction quality rather than measured robot task success.

\subsection{Decoding Sensitivity}
\label{app:decoding-sensitivity}
\input{tables/decoding_sensitivity}

\subsection{Motion and Dynamic Evaluation}
\label{app:motion-quality}

Figure~\ref{fig:motion-quality} uses the same 128 held-out trajectories per dataset and scores 32 predicted frames at $256\times320$ resolution. We use WorldArena's RAFT and DINO components~\citep{shang2026worldarena} to construct three reference-based diagnostics. Flow is the mean Euclidean difference between predicted-video and ground-truth optical-flow vectors over adjacent frames and spatial positions. Gap is the absolute difference between their mean flow magnitudes. DINO is the mean cosine similarity between corresponding predicted and ground-truth frame features. Each metric is averaged over trajectories; lower Flow and Gap and higher DINO indicate closer agreement with the recorded future. These diagnostics are supplementary measures rather than the official \mbox{WorldArena aggregate score.}

All metrics use GPU FP32 computation with TF32 disabled. FutureWorlds uses the main-comparison checkpoints: 500, 200 (motion-aware reward), and 400 MemSPO updates for RT-1, BridgeV2, and RoboCasa. RT-1 and RoboCasa videos are regenerated with the original decoding settings, while BridgeV2 retains the saved prediction tokens and decodes them uniformly on GPU. This supplementary evaluation does not replace the historical pixel-metric results. Ten- and twenty-frame evaluations use prefixes of the same rollouts. The qualitative cases are selected improved examples. For each case, the visualized pair has the largest ground-truth mean motion among the three stored pairs beginning at frames 8, 16, and 31. All methods share that pair and a color scale set by the pooled 99th-percentile flow magnitude; larger magnitudes are clipped. Hue represents direction and saturation represents magnitude. Values in the main figure are rounded to two decimal places.

\section{Short-Horizon Evaluation}
\label{app:short-horizon}
Table~\ref{tab:short-horizon-results} complements the 32-frame comparison with ten-frame results, using the same model checkpoints and 128 trajectories per reported dataset. Baseline scores use the first ten frames of their 32-frame rollouts; FutureWorlds uses its existing ten-frame evaluations. We average frame-wise PSNR, SSIM, and LPIPS-VGG over each trajectory and then over trajectories.

\input{tables/short_horizon_results}

\paragraph{Native interfaces.}
IRASim on BridgeV2 receives recorded commands, whereas its released checkpoint is trained with state-derived relative motion. Its RoboCasa result instead uses a target-domain 52-dimensional action adapter.

\Needspace{48\baselineskip}
\section{Policy Feedback in a Learned World Model}
\label{app:policy-feedback}

We examine a selected RT-1 drawer-closing example with a frozen, converged RT-1 policy. Starting from the same two recorded observations, the policy acts on each model's latest generated frame, and the world model predicts the next observation from that action. SFT, GRPO, and MemSPO share the initial policy state, random seed, beam width of four, and memory configuration; both post-trained models use 200 updates. No recorded future frame refreshes the interaction.

\begin{figure}[H]
  \centering
  \includegraphics[width=\linewidth]{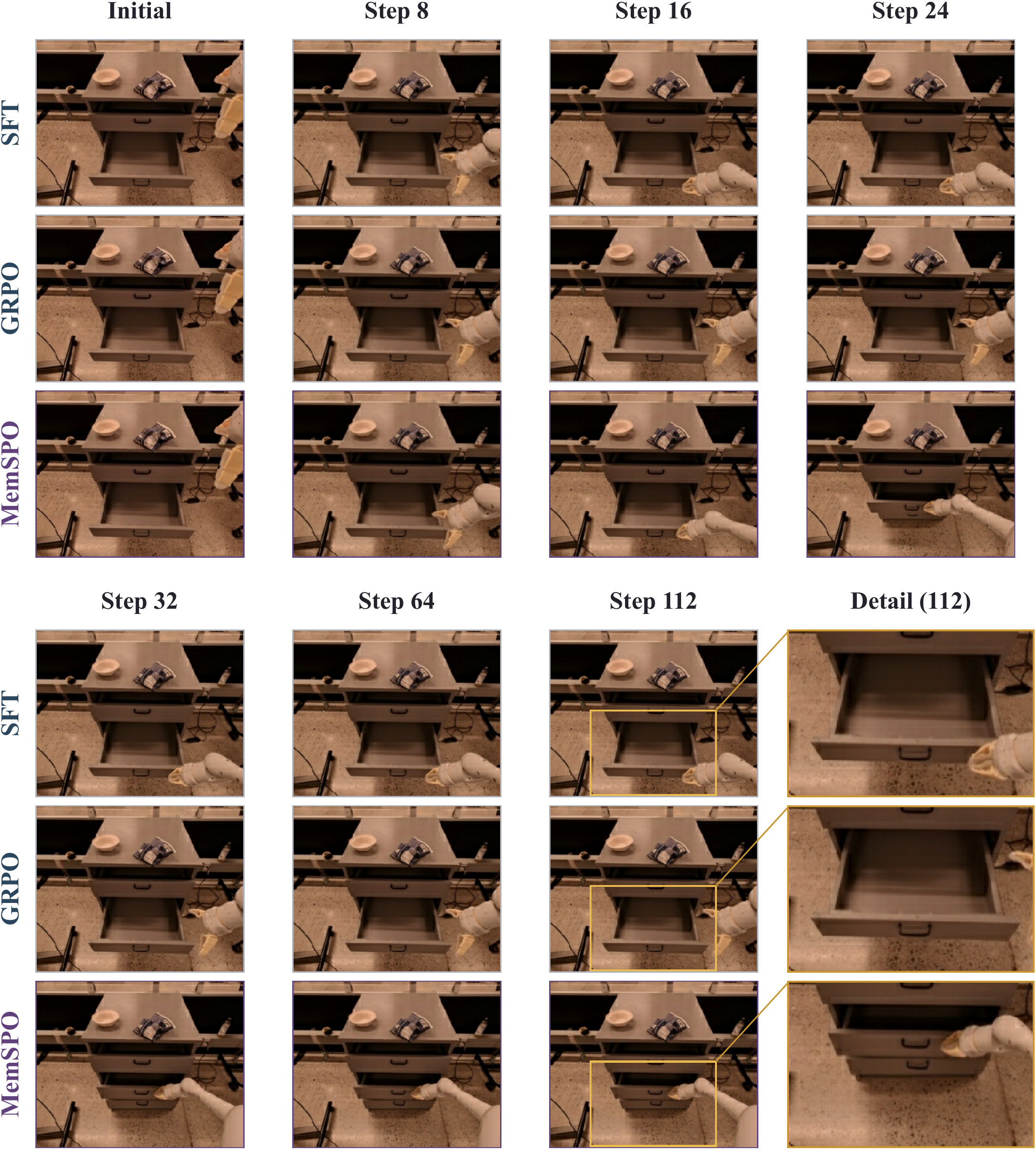}
  \setlength{\abovecaptionskip}{3pt}
  \caption{\textbf{Policy feedback in a learned world model.} A selected example for \textit{close middle drawer} with a fixed RT-1 policy. The two blocks show the same interaction at seven time points; gold boxes mark matched details at step 112.}
  \label{fig:policy-feedback}
\end{figure}

Figure~\ref{fig:policy-feedback} shows a visibly smaller drawer opening under MemSPO from step 32 onward, while SFT and GRPO retain larger openings. The contrast illustrates how post-training changes the scene evolution observed by the same policy during feedback interaction. This qualitative example concerns learned-model behavior and does not establish task success in a physical environment.

\clearpage
\section{Additional Qualitative Comparisons}
\label{app:additional-qualitative}

Figures~\ref{fig:additional-qualitative-1} and~\ref{fig:additional-qualitative-2} extend Figure~\ref{fig:rq1-qualitative} with six additional cases, two per dataset. We compare saved frame-32 predictions from WEAVER, Cosmos-Predict2.5-2B, DreamDojo-2B, and FutureWorlds against the same recorded future. Full images accompany detail views with identical crop coordinates across methods. These are selected positive examples; Table~\ref{tab:main-results} reports performance over the full evaluation cohorts.

\begin{figure}[H]
  \centering
  \includegraphics[width=\linewidth]{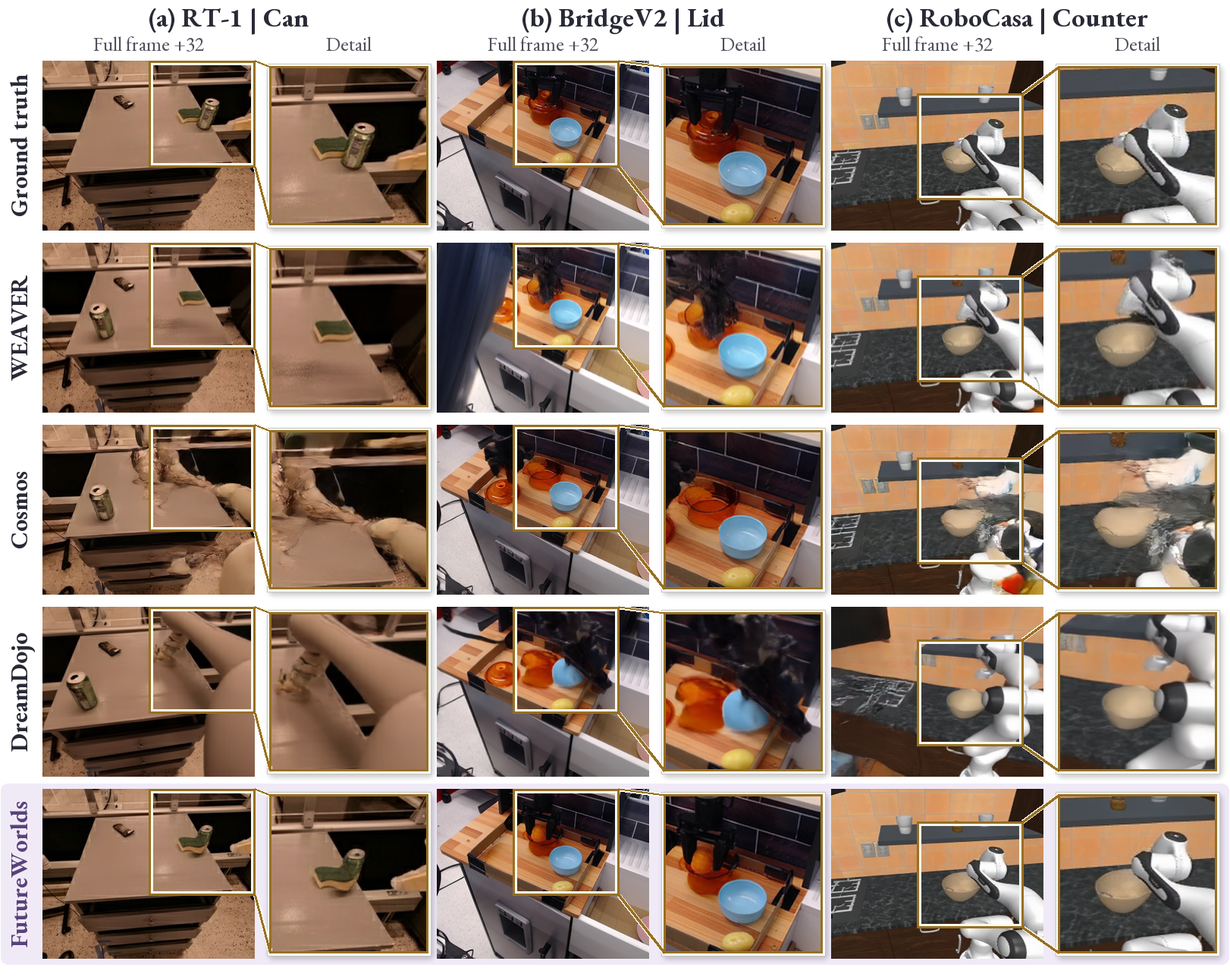}
  \setlength{\abovecaptionskip}{4pt}
  \caption{\textbf{\textit{Additional comparisons of object placement and robot pose.}} Selected RT-1, BridgeV2, and RoboCasa examples at frame 32. Gold boxes identify matched detail views; model names are abbreviated in the figure.}
  \label{fig:additional-qualitative-1}
\end{figure}

In the RT-1 case, FutureWorlds places the can nearer the sponge, as in the recorded future, while the shown baselines retain the can near its earlier position. The BridgeV2 example preserves a more recognizable lid--gripper configuration. In RoboCasa, the predicted arm configuration more closely resembles the recorded pose above the counter. These examples illustrate local visual agreement rather than robot task success.

\clearpage
\begin{figure}[H]
  \centering
  \includegraphics[width=\linewidth]{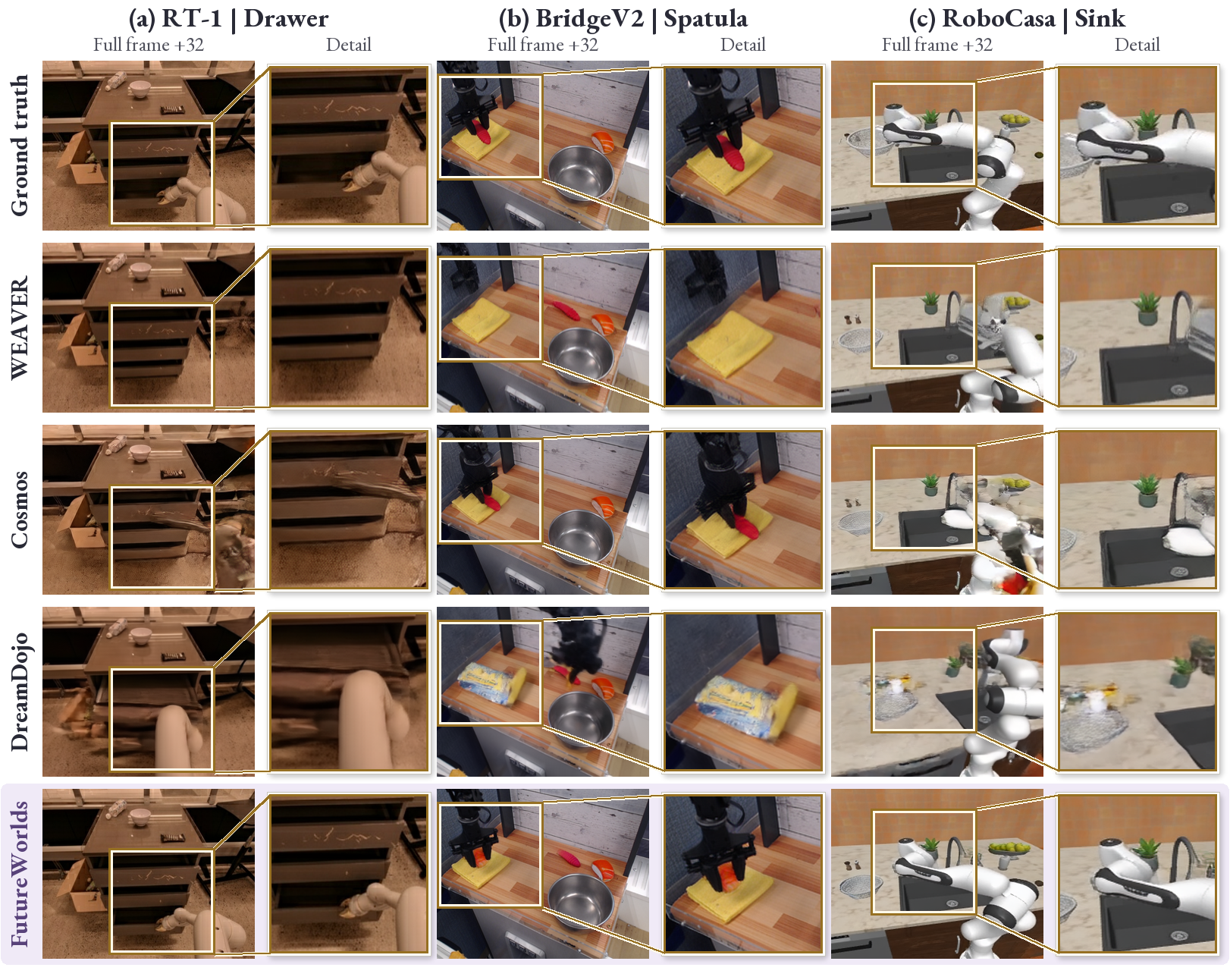}
  \setlength{\abovecaptionskip}{4pt}
  \caption{\textbf{\textit{Additional comparisons of articulated motion and object appearance.}} Three further selected cases at frame 32, using the same models and display convention as Figure~\ref{fig:additional-qualitative-1}.}
  \label{fig:additional-qualitative-2}
\end{figure}

The RT-1 example shows closer agreement in the drawer opening and gripper location. In BridgeV2, FutureWorlds retains the spatula near the yellow cloth, although its appearance and placement still differ from the reference; Cosmos also preserves the object in this example. The RoboCasa sink case shows a closer arm configuration than the displayed baselines. Full frames expose remaining differences outside the highlighted regions.

\clearpage
\subsection{Temporal Comparisons}
\label{app:temporal-comparisons}

Figures~\ref{fig:temporal-rt1}--\ref{fig:temporal-robocasa} show the evolution of three selected cases from the preceding comparisons. Each row follows one model at the same five future time points, with full frames retained throughout. The aligned sequences expose differences in motion timing, object displacement, and appearance over the rollout, complementing the frame-32 detail views.

\begin{figure}[H]
  \centering
  \includegraphics[width=\linewidth]{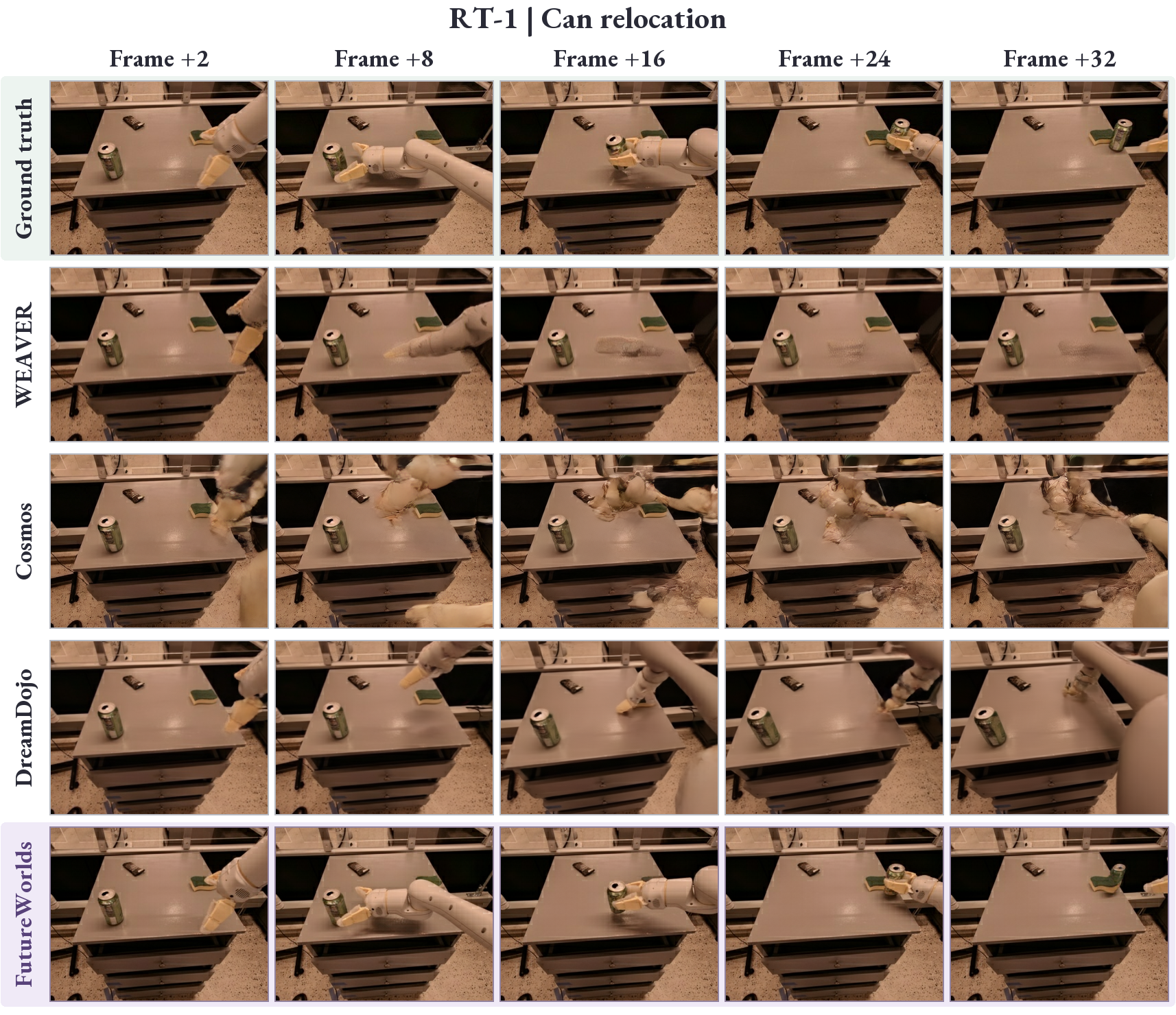}
  \setlength{\abovecaptionskip}{4pt}
  \caption{\textbf{\textit{Temporal comparison on RT-1.}} The instruction is \textit{move green can near sponge}. Columns show frames +2, +8, +16, +24, and +32 from the same saved prediction sequence for each method. FutureWorlds follows the can's displacement toward the sponge more closely, while the displayed baselines largely retain its earlier position.}
  \label{fig:temporal-rt1}
\end{figure}

\clearpage
\begin{figure}[H]
  \centering
  \includegraphics[width=\linewidth]{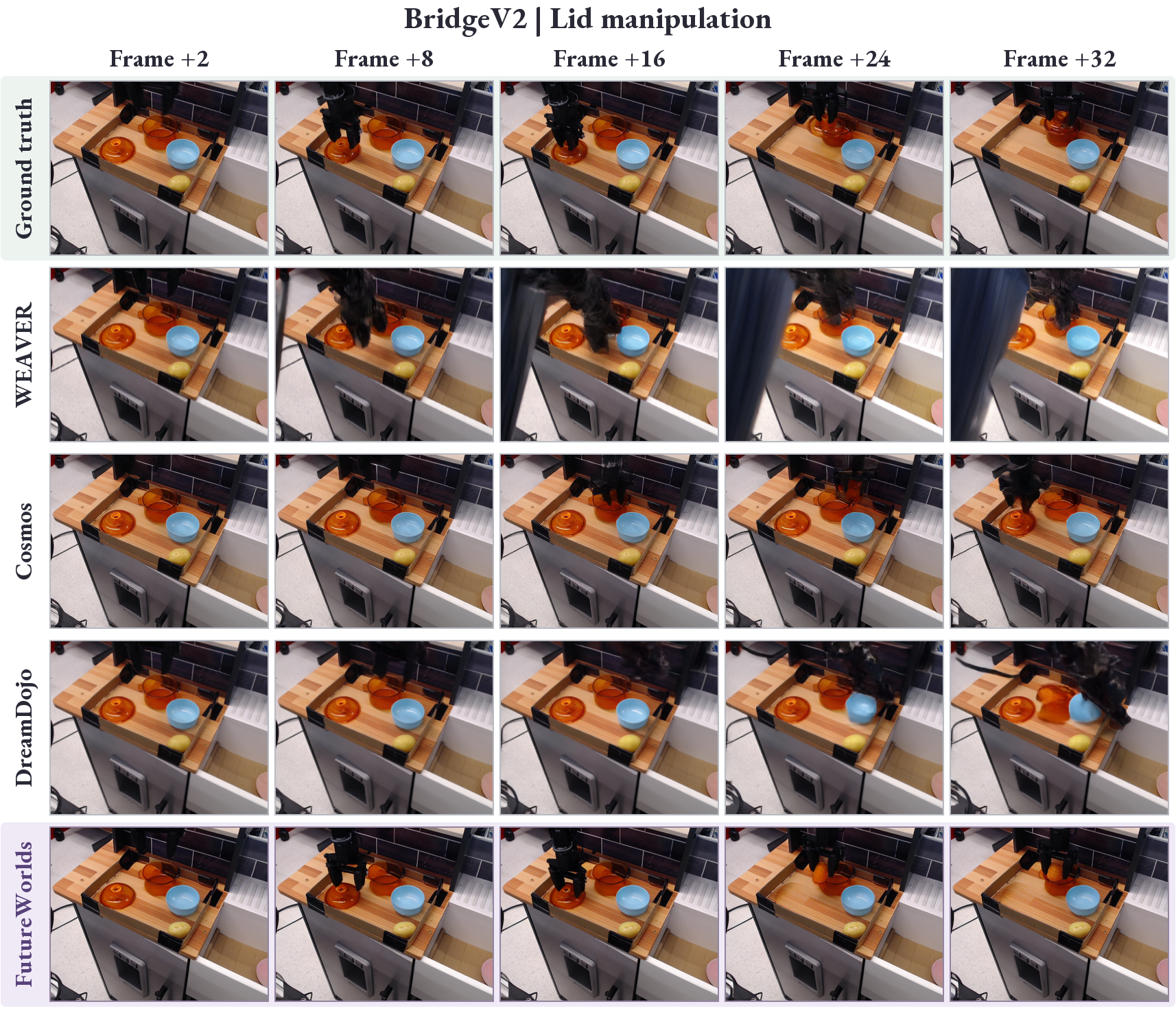}
  \setlength{\abovecaptionskip}{4pt}
  \caption{\textbf{\textit{Temporal comparison on BridgeV2.}} A selected lid-manipulation sequence at the same five time points as Figure~\ref{fig:temporal-rt1}. Full frames reveal the evolving gripper--object configuration and remaining differences from the recorded trajectory.}
  \label{fig:temporal-bridge}
\end{figure}

\clearpage
\begin{figure}[H]
  \centering
  \includegraphics[width=\linewidth]{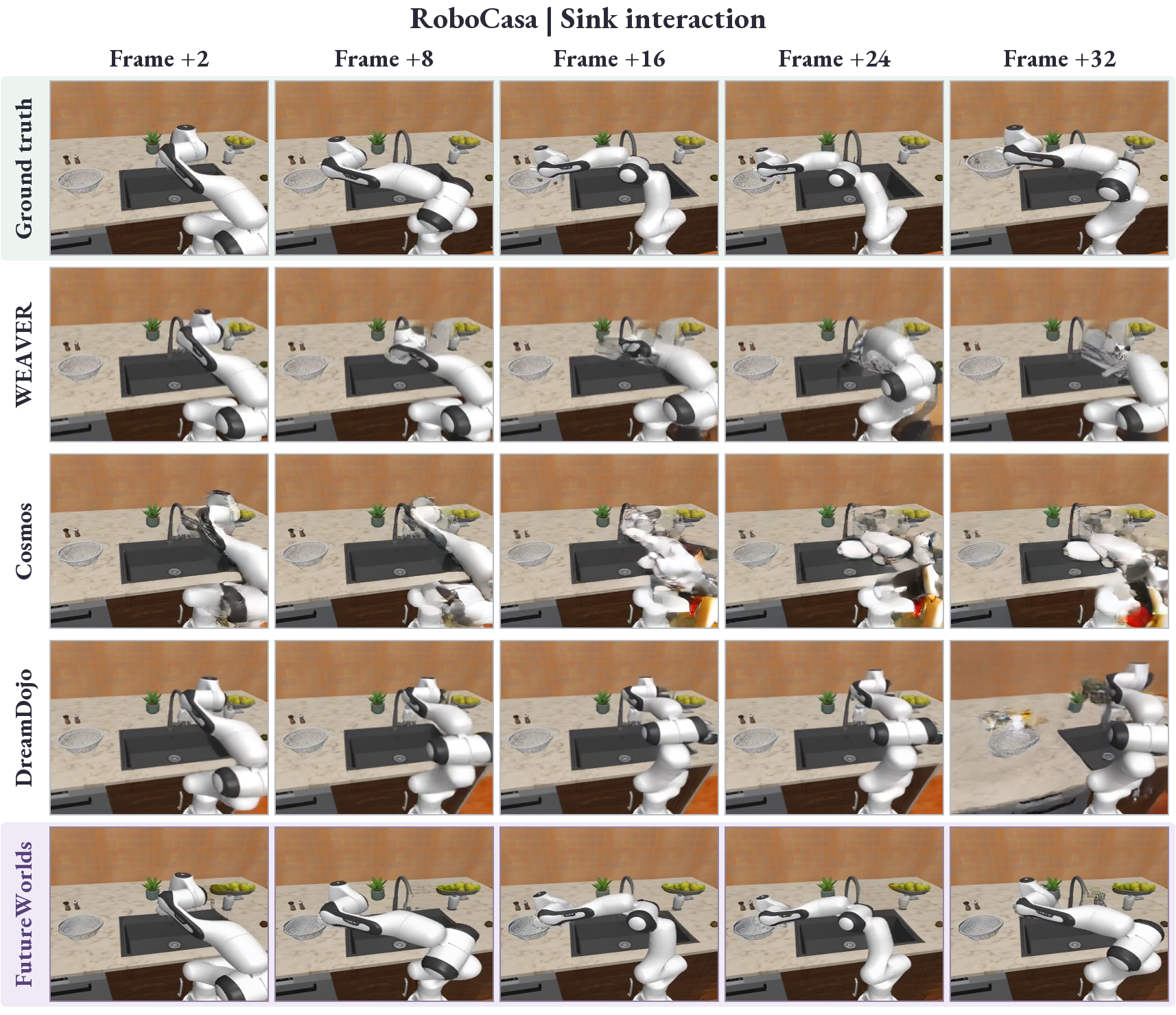}
  \setlength{\abovecaptionskip}{4pt}
  \caption{\textbf{\textit{Temporal comparison on RoboCasa.}} A selected sink-interaction sequence under the same display convention. FutureWorlds more closely follows the changing arm configuration, while deviations in pose, geometry, and scene appearance are visible in the baseline sequences. These selected examples illustrate prediction behavior, not measured task success.}
  \label{fig:temporal-robocasa}
\end{figure}

\clearpage
\subsection{Candidate Diversity and Relative Learning Signals}
\label{app:candidate-diversity}

We inspect saved candidates from the first BridgeV2 training batch, before either strategy updates the shared SFT checkpoint. Both strategies use the same inputs, frozen tokenizer, ten-frame horizon, and reward. Ordinary beam search maintains eight global paths and retains the top four; diverse search uses four groups of two with a diversity penalty of $0.1$, retaining one candidate per group.

\begin{figure}[H]
  \centering
  \includegraphics[width=\linewidth]{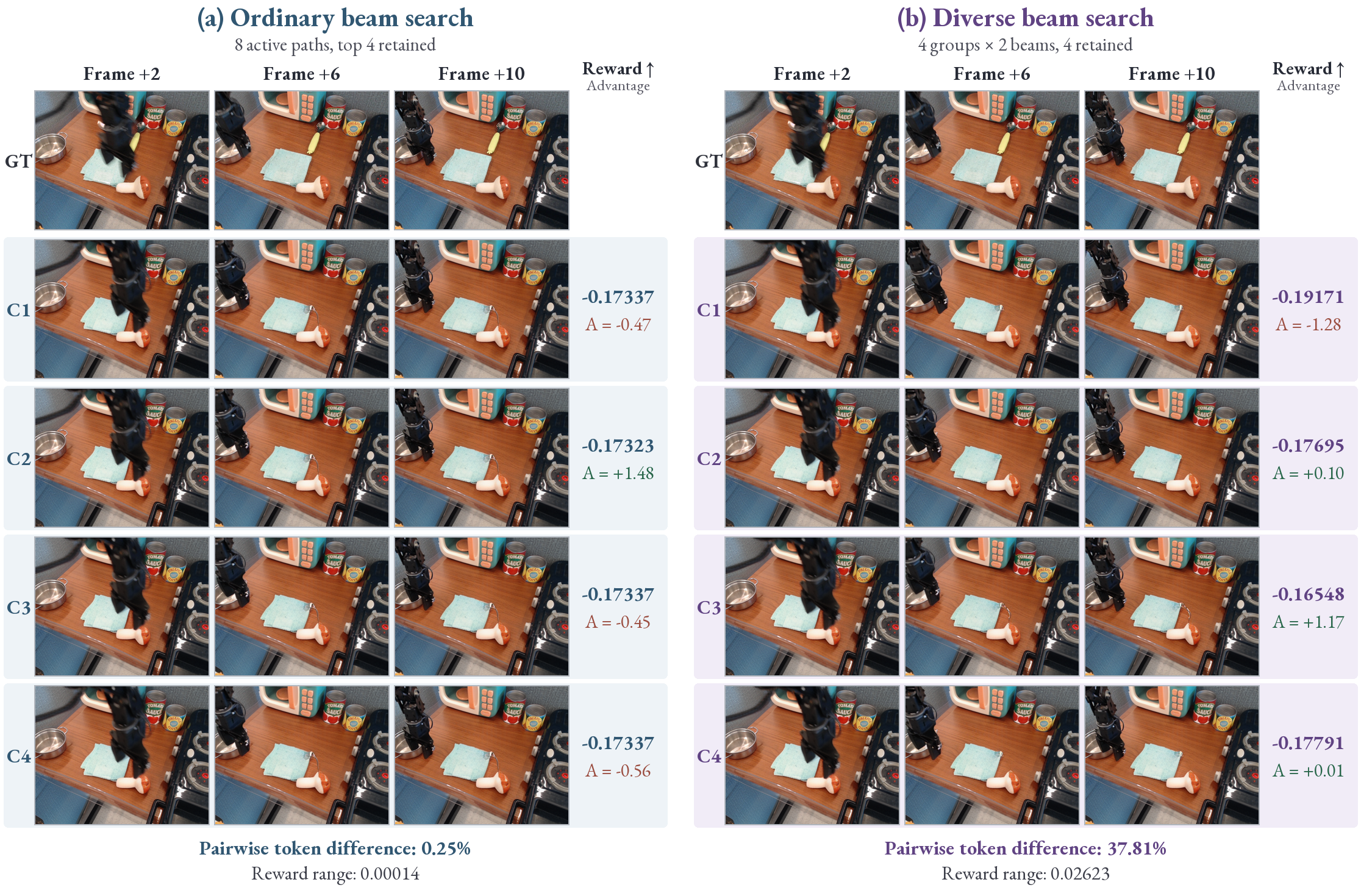}
  \setlength{\abovecaptionskip}{4pt}
  \caption{\textbf{\textit{Candidate trajectories and relative learning signals.}} Ordinary and diverse beam search from the same SFT model and input. All four retained candidates appear in their original order at frames +2, +6, and +10. Annotations give the original ten-frame reward and group-relative advantage $A$. The illustrated input has the largest diverse-search reward range among the 16 inputs in the initial batch.}
  \label{fig:beam-candidates}
\end{figure}

Ordinary search produces nearly identical visible continuations, whereas diverse search exposes differences in gripper configuration and object appearance. Positive and negative advantages give the RL objective alternatives to promote or suppress. Across all 16 inputs in this initial batch, mean pairwise token disagreement is $0.28\%$ for ordinary search and $30.32\%$ for diverse search; the mean within-group reward range is $6.44\!\times\!10^{-5}$ and $8.66\!\times\!10^{-3}$, respectively. Token disagreement averages over all six candidate pairs and $10\times80$ token positions; it is distinct from perceptual diversity.

Diverse search also includes weaker candidates: its mean reward is $-0.14353$, compared with $-0.14041$ for ordinary search. These diagnostics illustrate reduced candidate redundancy and a broader quality distribution, rather than an immediate increase in average quality. Since advantages are group-normalized, a wider raw reward range alone does not establish stronger gradients or improved post-training performance. Table~\ref{tab:ordinary-beam-posttraining} complements these candidate-level diagnostics with final-model results.

\input{sections/ordinary_beam_appendix}

\input{sections/memory_appendix}

\input{sections/resource_appendix}

%% file: tables/decoding_sensitivity.tex
% Matched native-beam evaluation: 32 fixed clips/dataset, H10, width four.
% Values: mean LPIPS relative reduction (%) from matched SFT to MemSPO200.
\begin{table}[H]
\caption{LPIPS reduction (\%) across decoding settings.}
\label{tab:decoding-sensitivity}
\centering
\begingroup
\fontencoding{T1}\fontfamily{EBGaramond-TLF}\selectfont
\fontsize{9.5}{11}\selectfont
\definecolor{sensitivityhead}{HTML}{EAE6F2}
\definecolor{sensitivitystripe}{HTML}{F6F4F9}
\definecolor{sensitivityink}{HTML}{544178}
\setlength{\tabcolsep}{4pt}
\renewcommand{\arraystretch}{1.04}
\begin{tabularx}{\linewidth}{@{}>{\centering\arraybackslash}X>{\centering\arraybackslash}X>{\centering\arraybackslash}X>{\centering\arraybackslash}X@{}}
\toprule
\rowcolor{sensitivityhead}
\textbf{Temperature $T$} & \textbf{Top-$p$} & \textbf{RT-1} & \textbf{BridgeV2}\\
\midrule
\rowcolor{sensitivitystripe}
1.0 & 1.00 & \textcolor{sensitivityink}{4.02} & \textcolor{sensitivityink}{3.45}\\
0.8 & 1.00 & 3.03 & 2.97\\
1.2 & 1.00 & 4.32 & 2.60\\
1.0 & 0.95 & 3.70 & 3.84\\
1.0 & 0.90 & 4.04 & 3.53\\
\bottomrule
\end{tabularx}
\endgroup
\end{table}

\noindent\textbf{\textit{Decoding sensitivity.}}
Table~\ref{tab:decoding-sensitivity} varies temperature $T$ and the nucleus threshold $p$, keeping the beam width at four. On the same 32 clips per dataset, MemSPO reduces mean LPIPS under all five settings: 3.03--4.32\% on RT-1 and 2.60--3.84\% on BridgeV2. Each setting uses its own matched SFT reference, so the gains do not depend solely on the default decoding parameters.

%% file: tables/short_horizon_results.tex
\begingroup
\definecolor{resultink}{HTML}{292C42}
\definecolor{resultpurple}{HTML}{544178}
\definecolor{resultheader}{HTML}{E9E8F2}
\definecolor{resultstripe}{HTML}{F6F7FA}
\definecolor{resultours}{HTML}{EEE7F7}
\definecolor{rankfirst}{HTML}{E4D9F1}
\definecolor{ranksecond}{HTML}{F7E7D6}
\definecolor{rankthird}{HTML}{E0EDF4}
\definecolor{resultpending}{HTML}{7A7D8C}
\newcommand{\pendingresult}{\textcolor{resultpending}{\textendash}}
\newcommand{\pendingdataset}{\pendingresult & \pendingresult & \pendingresult}
\newcommand{\pendingrow}{\pendingdataset & \pendingdataset & \pendingdataset}
\newcommand{\resultcategory}[1]{%
  \addlinespace[2pt]
  \multicolumn{10}{@{}l}{\rule{0pt}{11pt}\textcolor{resultpurple}{\bfseries\itshape #1}}\\[1pt]}
\newcommand{\resultmetrics}{%
  {\fontsize{8.5}{9.6}\selectfont\mbox{PSNR}} &
  {\fontsize{8.5}{9.6}\selectfont\mbox{SSIM}} &
  {\fontsize{8.5}{9.6}\selectfont\mbox{LPIPS}}}

\begin{table}[H]
\centering
\caption{Ten-frame video prediction on 128 trajectories per reported dataset. Metrics: PSNR$\uparrow$, SSIM$\uparrow$, and LPIPS$\downarrow$.}
\label{tab:short-horizon-results}
\vspace{3pt}
\fontencoding{T1}\fontfamily{EBGaramond-TLF}\selectfont
\fontsize{9.0}{10.5}\selectfont
\color{resultink}
\setlength{\tabcolsep}{1.2pt}
\renewcommand{\arraystretch}{1.12}
\setlength{\aboverulesep}{2.4pt}
\setlength{\belowrulesep}{2.4pt}
% Only method names determine the first-column width; categories span all 10 columns.
\begin{tabularx}{\linewidth}{@{}p{0.25\linewidth}*{9}{>{\centering\arraybackslash}X}@{}}
\toprule[0.95pt]
\rowcolor{resultheader}
\rule{0pt}{10.5pt}\textbf{Method}
 & \multicolumn{3}{c}{\textbf{RT-1}}
 & \multicolumn{3}{c}{\textbf{BridgeV2}}
 & \multicolumn{3}{c@{}}{\textbf{RoboCasa}}\\[-0.5pt]
\cmidrule(lr){2-4}\cmidrule(lr){5-7}\cmidrule(l){8-10}
 & \resultmetrics & \resultmetrics & \resultmetrics\\[-0.5pt]
\midrule[0.55pt]
\resultcategory{Video Generation Models}
\rowcolor{resultstripe} FitVid & 19.89 & 0.6268 & 0.4529 & 17.86 & 0.6241 & 0.5079 & 14.29 & 0.5640 & 0.5373\\
iVideoGPT & 16.62 & 0.6142 & 0.3333 & 21.91 & 0.8142 & 0.1544 & 17.56 & 0.7513 & 0.2009\\
\resultcategory{Action-Conditioned World Models}
\rowcolor{resultstripe} IRASim & 18.64 & 0.6771 & 0.2747 & 18.65 & 0.7442 & 0.2120 & 14.75 & 0.6495 & 0.3069\\
Vid2World & 20.82 & 0.7803 & 0.1551 & 21.21 & 0.7648 & 0.1482 & 16.62 & 0.7452 & 0.1873\\
\rowcolor{resultstripe} Ctrl-World & 20.67 & 0.7548 & 0.2094 & 20.84 & 0.7456 & 0.2071 & 16.27 & 0.7335 & 0.2325\\
DriftWorld & 23.14 & 0.8185 & 0.1861 & 22.72 & 0.8103 & 0.1699 & 17.93 & 0.7746 & 0.2203\\
\rowcolor{resultstripe} WEAVER & 21.12 & 0.8038 & 0.1932 & 21.20 & 0.7929 & 0.1993 & 15.84 & 0.7523 & 0.2308\\
\resultcategory{RL Post-Training}
\rowcolor{resultstripe} RLVR-World & 21.24 & 0.7937 & 0.1447 & 21.33 & 0.7702 & 0.1398 & 16.98 & 0.7551 & 0.1652\\
PersistWorld & 20.56 & 0.7522 & 0.2154 & 21.62 & 0.7646 & 0.1907 & 16.97 & 0.7481 & 0.2166\\
\resultcategory{Adapted Video Foundation Models}
\rowcolor{resultstripe} Wan2.2-TI2V-5B & 20.12 & 0.7707 & 0.1892 & 19.03 & 0.7784 & 0.1737 & 15.53 & 0.7388 & 0.2165\\
Cosmos-Predict2.5-2B & 20.12 & 0.7746 & 0.2035 & 21.36 & 0.8365 & 0.1301 & 15.51 & 0.7382 & 0.2268\\
\rowcolor{resultstripe} DreamDojo-2B & 20.57 & 0.8197 & 0.1892 & 19.31 & 0.8190 & 0.1808 & 15.10 & 0.7454 & 0.2322\\
\addlinespace[3pt]\midrule[0.6pt]
\rowcolor{resultours}\rule{0pt}{13pt}\textcolor{resultpurple}{\bfseries FutureWorlds (ours)} & \textbf{24.21} & \textbf{0.8345} & \textbf{0.1252} & \textbf{24.34} & \textbf{0.8318} & \textbf{0.1072} & \textbf{20.18} & \textbf{0.8045} & \textbf{0.1434}\\[2pt]
\bottomrule[0.95pt]
\end{tabularx}
\end{table}
\endgroup

%% file: sections/ordinary_beam_appendix.tex
\clearpage
\subsection{Candidate Construction and Final Post-Training Quality}
\label{app:ordinary-beam-posttraining}

\noindent\textbf{\emph{Matched training and evaluation.}}
All RL methods start from the same dataset-specific SFT checkpoint at 50,000 updates and use 200 post-training updates, batch size 16, ten-frame rollouts, and the same first 3,200 training windows in order. We fix the base reward (Appendix~\ref{app:reward-design}), optimizer settings, reference policy, visual tokenizer, memory, and history-matched scoring. Ordinary beam search maintains eight global paths and returns the top four complete trajectories by cumulative model log probability. MemSPO uses four groups of two paths with a diversity penalty of $0.1$, returning one trajectory per group; GRPO samples four independent trajectories. At evaluation, all methods use beam width four, two initial observations, and anchor-plus-six-state memory without future ground-truth refreshes. The 10/20-frame scores use prefixes of the same 32-frame predictions.

\input{tables/ordinary_beam_posttraining}

\noindent\textbf{\emph{Evaluation consistency.}}
Every row uses GPU FP32 with TF32 disabled and the same ordered 128 cases within its dataset. BridgeV2 controls are reevaluated together with ordinary beam search; Table~\ref{tab:matched-post-training} retains its earlier CPU-generated BridgeV2 results. Checkpoints and inputs match, but CPU/GPU numerical differences can alter discrete beam decisions and subsequent rollouts, so these values are not interchangeable. RT-1 and RoboCasa reuse their verified matched GPU controls. Comparisons in this table use only its internally matched results.

\noindent\textbf{\emph{Final-model evidence.}}
MemSPO improves PSNR and LPIPS over ordinary-beam post-training on all three datasets at every tested horizon. At 32 frames, LPIPS falls by 7.37\%, 0.20\%, and 1.51\% on RT-1, BridgeV2, and RoboCasa, respectively. The advantage is metric-dependent: ordinary beam search has lower BridgeV2 MAE at 32 frames (0.042731 versus 0.042887). Together with Figure~\ref{fig:beam-candidates}, the final-model results support the usefulness of less redundant candidate construction for post-training. This compares global and grouped diverse search, including their different beam allocation; it does not isolate the diversity coefficient or establish statistical significance.

%% file: tables/ordinary_beam_posttraining.tex
\begingroup
\definecolor{beampurple}{HTML}{544178}
\definecolor{beamheader}{HTML}{E9E8F2}
\definecolor{beamstripe}{HTML}{F6F7FA}
\definecolor{beamours}{HTML}{EEE7F7}
\begin{table}[H]
\centering
\caption{\textbf{Post-training with different candidate construction strategies.} Matched GPU evaluation on 128 trajectories per dataset. PSNR$\uparrow$, SSIM$\uparrow$, and LPIPS$\downarrow$; bold marks the best score. Reduction rows compare MemSPO with ordinary beam search using unrounded LPIPS.}
\label{tab:ordinary-beam-posttraining}
\vspace{2pt}
\fontencoding{T1}\fontfamily{EBGaramond-TLF}\selectfont
\fontsize{10}{11.5}\selectfont
\setlength{\tabcolsep}{1.2pt}
\renewcommand{\arraystretch}{1.0}
\begin{tabularx}{\linewidth}{@{}p{0.23\linewidth}*{9}{>{\centering\arraybackslash}X}@{}}
\toprule[0.9pt]
\rowcolor{beamheader}
\textbf{Method (updates)} & \multicolumn{3}{c}{\textbf{10 frames}} & \multicolumn{3}{c}{\textbf{20 frames}} & \multicolumn{3}{c@{}}{\textbf{32 frames}}\\
\cmidrule(lr){2-4}\cmidrule(lr){5-7}\cmidrule(l){8-10}
 & PSNR & SSIM & LPIPS & PSNR & SSIM & LPIPS & PSNR & SSIM & LPIPS\\
\midrule[0.5pt]
\multicolumn{10}{@{}l}{\rule{0pt}{12pt}\textcolor{beampurple}{\bfseries\itshape RT-1}}\\[0pt]
\rowcolor{beamstripe} SFT (50k) & 23.76 & 0.8298 & 0.1308 & 22.49 & 0.8075 & 0.1527 & 21.37 & 0.7860 & 0.1734\\[1pt]
GRPO (200) & 23.97 & 0.8319 & 0.1289 & 22.60 & 0.8089 & 0.1517 & 21.39 & 0.7858 & 0.1739\\[1pt]
\rowcolor{beamstripe} Ordinary beam (200) & 24.03 & 0.8329 & 0.1279 & 22.85 & 0.8124 & 0.1479 & 21.70 & 0.7908 & 0.1690\\[1pt]
\rowcolor{beamours} \textcolor{beampurple}{\bfseries MemSPO (200)} & \textbf{24.33} & \textbf{0.8357} & \textbf{0.1240} & \textbf{23.42} & \textbf{0.8192} & \textbf{0.1397} & \textbf{22.45} & \textbf{0.8016} & \textbf{0.1565}\\[1pt]
\textcolor{beampurple}{\itshape LPIPS reduction} & \multicolumn{3}{c}{\textcolor{beampurple}{3.04\%}} & \multicolumn{3}{c}{\textcolor{beampurple}{5.54\%}} & \multicolumn{3}{c}{\textcolor{beampurple}{7.37\%}}\\[1pt]
\noalign{\vskip 2pt\hbox to\linewidth{\color{gray}\leaders\hbox{\rule{2pt}{0.35pt}\hskip 2pt}\hfill}\vskip 1pt}
\multicolumn{10}{@{}l}{\rule{0pt}{12pt}\textcolor{beampurple}{\bfseries\itshape BridgeV2}}\\[0pt]
\rowcolor{beamstripe} SFT (50k) & 24.06 & 0.8300 & 0.1095 & 22.74 & 0.8085 & 0.1270 & 21.74 & 0.7942 & 0.1448\\[1pt]
GRPO (200) & 24.11 & 0.8300 & 0.1090 & 22.78 & 0.8086 & 0.1266 & 21.79 & 0.7944 & 0.1442\\[1pt]
\rowcolor{beamstripe} Ordinary beam (200) & 24.08 & 0.8298 & 0.1095 & 22.86 & 0.8091 & 0.1259 & 21.89 & 0.7956 & 0.1424\\[1pt]
\rowcolor{beamours} \textcolor{beampurple}{\bfseries MemSPO (200)} & \textbf{24.31} & \textbf{0.8320} & \textbf{0.1073} & \textbf{22.99} & \textbf{0.8108} & \textbf{0.1244} & \textbf{21.95} & \textbf{0.7960} & \textbf{0.1421}\\[1pt]
\textcolor{beampurple}{\itshape LPIPS reduction} & \multicolumn{3}{c}{\textcolor{beampurple}{1.99\%}} & \multicolumn{3}{c}{\textcolor{beampurple}{1.18\%}} & \multicolumn{3}{c}{\textcolor{beampurple}{0.20\%}}\\[1pt]
\noalign{\vskip 2pt\hbox to\linewidth{\color{gray}\leaders\hbox{\rule{2pt}{0.35pt}\hskip 2pt}\hfill}\vskip 1pt}
\multicolumn{10}{@{}l}{\rule{0pt}{12pt}\textcolor{beampurple}{\bfseries\itshape RoboCasa}}\\[0pt]
\rowcolor{beamstripe} SFT (50k) & 19.50 & 0.7957 & 0.1513 & 18.22 & 0.7689 & 0.1787 & 17.50 & 0.7519 & 0.1971\\[1pt]
GRPO (200) & 19.68 & 0.7965 & 0.1514 & 18.64 & 0.7729 & 0.1758 & 18.03 & 0.7581 & 0.1919\\[1pt]
\rowcolor{beamstripe} Ordinary beam (200) & 19.69 & \textbf{0.7985} & 0.1490 & 18.57 & \textbf{0.7746} & 0.1745 & 17.98 & \textbf{0.7607} & 0.1906\\[1pt]
\rowcolor{beamours} \textcolor{beampurple}{\bfseries MemSPO (200)} & \textbf{19.92} & 0.7984 & \textbf{0.1476} & \textbf{18.88} & 0.7743 & \textbf{0.1718} & \textbf{18.29} & 0.7600 & \textbf{0.1877}\\[1pt]
\textcolor{beampurple}{\itshape LPIPS reduction} & \multicolumn{3}{c}{\textcolor{beampurple}{0.94\%}} & \multicolumn{3}{c}{\textcolor{beampurple}{1.53\%}} & \multicolumn{3}{c}{\textcolor{beampurple}{1.51\%}}\\[1pt]
\bottomrule[0.9pt]
\end{tabularx}
\end{table}
\endgroup

%% file: sections/memory_appendix.tex
\clearpage
\section{Memory Ablation: Protocol and Full Results}
\label{app:memory-ablation}

\noindent\textbf{\emph{Controlled inference-time intervention.}}
We hold the MemSPO checkpoint at 200 updates, beam width at four, and the ordered set of 128 held-out trajectories per dataset fixed. Actions, language conditioning where applicable, and the visual tokenizer are identical across memory variants. Only the retained state history changes; the variants do not undergo separate training. Each evaluation generates 32 frames once, with its first 10 and 20 frames providing the shorter-horizon results. Evaluation uses GPU FP32 with TF32 disabled, and no ground-truth future frame refreshes the rollout. These settings reproduce the comparison in Figure~\ref{fig:memory-ablation}; the main-table checkpoints use different post-training budgets and reward configurations as detailed in Appendix~\ref{app:reward-design}.

\noindent\textbf{\emph{Memory configurations.}}
\textit{w/o Historical Memory} retains only the latest state frame, excluding earlier dynamic history.
\textit{w/o Initial Anchor} retains a sliding window of the six most recent state frames without a persistent initial state anchor.
\textit{Full Model} retains the initial state anchor alongside the six most recent states.
All three variants retain the same fixed initial scene encoding; removing historical memory therefore does not remove all initial visual information. State frames retain their corresponding action tokens, and the action indexing and causal prediction order remain unchanged.

\input{tables/memory_ablation}

\noindent\textbf{\emph{Results and interpretation.}}
Table~\ref{tab:memory-full} reports the three metrics shown in Figure~\ref{fig:memory-ablation} at all evaluated horizons. Retaining recent history yields the largest improvement over removing historical memory on every dataset. At 32 frames, the full model achieves the best PSNR, SSIM, and MSE, reducing MSE by 40.83\%, 57.69\%, and 29.22\% relative to the history-removed variant on RT-1, BridgeV2, and RoboCasa. The initial anchor provides smaller additional gains. Its effect varies with horizon and metric: on RoboCasa, removing the anchor yields slightly lower MSE at 10 and 20 frames, while the full model retains higher PSNR and SSIM. These results distinguish the substantial role of recent context from the more modest contribution of a persistent anchor. We report cohort means without inferring statistical significance.

%% file: tables/memory_ablation.tex
% Generated from the verified 128-trajectory memory ablation aggregates.
\begingroup
\definecolor{mempurple}{HTML}{544178}
\definecolor{memheader}{HTML}{E9E8F2}
\definecolor{memstripe}{HTML}{F6F7FA}
\definecolor{memours}{HTML}{EEE7F7}
\begin{table}[H]
\centering
\caption{\textbf{Memory ablation across prediction horizons.} PSNR$\uparrow$ (dB), SSIM$\uparrow$ ($\times100$), and MSE$\downarrow$ ($\times10^3$), averaged over 128 trajectories per dataset. Bold marks the best unrounded value within each dataset and horizon; display values use two decimals.}
\label{tab:memory-full}
\vspace{3pt}
\fontencoding{T1}\fontfamily{EBGaramond-TLF}\selectfont
\fontsize{10.5}{12}\selectfont
\setlength{\tabcolsep}{1.2pt}
\renewcommand{\arraystretch}{1.05}
\begin{tabularx}{\linewidth}{@{}p{0.25\linewidth}*{9}{>{\centering\arraybackslash}X}@{}}
\toprule[0.9pt]
\rowcolor{memheader}
\textbf{Memory configuration} & \multicolumn{3}{c}{\textbf{10 frames}} & \multicolumn{3}{c}{\textbf{20 frames}} & \multicolumn{3}{c@{}}{\textbf{32 frames}}\\
\cmidrule(lr){2-4}\cmidrule(lr){5-7}\cmidrule(l){8-10}
 & PSNR & SSIM & MSE & PSNR & SSIM & MSE & PSNR & SSIM & MSE\\
\midrule[0.5pt]
\multicolumn{10}{@{}l}{\rule{0pt}{13pt}\textcolor{mempurple}{\bfseries\itshape RT-1}}\\[1pt]
\rowcolor{memstripe} w/o Historical Memory & 22.45 & 81.81 & 7.37 & 20.81 & 78.83 & 11.00 & 19.89 & 76.73 & 13.21\\[1pt]
w/o Initial Anchor & 24.32 & 83.56 & 4.57 & 23.39 & 81.85 & 6.12 & 22.43 & 80.07 & 7.99\\[1pt]
\rowcolor{memours} \textcolor{mempurple}{\bfseries Full Model} & \textbf{24.33} & \textbf{83.57} & \textbf{4.56} & \textbf{23.42} & \textbf{81.92} & \textbf{6.01} & \textbf{22.45} & \textbf{80.16} & \textbf{7.82}\\[1pt]
\noalign{\vskip 3pt\hbox to\linewidth{\color{gray}\leaders\hbox{\rule{2pt}{0.35pt}\hskip 2pt}\hfill}\vskip 2pt}
\multicolumn{10}{@{}l}{\rule{0pt}{13pt}\textcolor{mempurple}{\bfseries\itshape BridgeV2}}\\[1pt]
\rowcolor{memstripe} w/o Historical Memory & 20.04 & 79.82 & 14.38 & 18.63 & 77.26 & 18.34 & 17.95 & 75.67 & 20.51\\[1pt]
w/o Initial Anchor & 24.29 & 83.17 & 4.94 & 22.94 & 81.01 & 6.99 & 21.93 & 79.58 & 8.70\\[1pt]
\rowcolor{memours} \textcolor{mempurple}{\bfseries Full Model} & \textbf{24.31} & \textbf{83.20} & \textbf{4.91} & \textbf{22.99} & \textbf{81.08} & \textbf{6.90} & \textbf{21.95} & \textbf{79.60} & \textbf{8.68}\\[1pt]
\noalign{\vskip 3pt\hbox to\linewidth{\color{gray}\leaders\hbox{\rule{2pt}{0.35pt}\hskip 2pt}\hfill}\vskip 2pt}
\multicolumn{10}{@{}l}{\rule{0pt}{13pt}\textcolor{mempurple}{\bfseries\itshape RoboCasa}}\\[1pt]
\rowcolor{memstripe} w/o Historical Memory & 18.21 & 77.28 & 20.92 & 17.21 & 74.69 & 26.63 & 16.70 & 73.18 & 29.88\\[1pt]
w/o Initial Anchor & 19.91 & 79.83 & \textbf{13.77} & 18.85 & 77.39 & \textbf{18.34} & 18.24 & 75.91 & 21.18\\[1pt]
\rowcolor{memours} \textcolor{mempurple}{\bfseries Full Model} & \textbf{19.92} & \textbf{79.84} & 13.79 & \textbf{18.88} & \textbf{77.43} & 18.35 & \textbf{18.29} & \textbf{76.00} & \textbf{21.15}\\[1pt]
\bottomrule[0.9pt]
\end{tabularx}
\end{table}
\endgroup

%% file: sections/resource_appendix.tex
\clearpage
\section{Inference Resource Measurements}
\label{app:resource-measurements}

\noindent\textbf{\emph{Benchmark scope.}}
Table~\ref{tab:resource-measurements} complements Figure~\ref{fig:efficiency-bubble} with parameters, memory, and latency for all seven plotted models. Each model uses one NVIDIA H20 on the same node, batch size one, and the first three fixed trajectories in the BridgeV2 evaluation cohort. Models run on separate GPUs rather than sharing a device. Each returns 32 RGB frames at $256\times320$ resolution, using its native checkpoint, internal resolution, input interface, and decoding settings. Costs use three cases, whereas plotted quality uses all 128 evaluation trajectories.

\input{tables/resource_measurements}

\noindent\textbf{\emph{Parameter accounting.}}
Inference parameters include the prediction backbone, action/history adapters, and visual encoder--decoder, counting tied tensors once. Offline language encoders and unused task, reward, or critic modules are excluded. Loaded parameters additionally include retained unused modules inventoried in the native pipeline: WEAVER retains a task encoder and reward/critic networks, while PersistWorld retains unused image and text encoders. FutureWorlds contains 127,101,696 backbone parameters and 129,448,016 visual-codec parameters, totaling 256,549,712. Its frozen RL reference and training reward network are absent from inference.

\noindent\textbf{\emph{Warmup and timing.}}
We preload input image/action arrays and model weights, disable gradients and TF32, and perform one complete 32-frame warmup. We then time each of the three cases once using a wall-clock timer, synchronizing CUDA immediately before and after inference. The table reports the median and observed range, not a confidence interval. Timing includes input preprocessing, CPU--GPU transfers, visual encoding, prediction, decoding, resizing, and transfer of the output to CPU. Model loading, trajectory-file reading, video saving, and quality metrics are excluded. Text features are precomputed. Wan loads and transfers its cached instruction embedding inside the timed call; this cache-file access is included, while language encoding is excluded.

\noindent\textbf{\emph{Memory and residency.}}
Before each measured call, we reset PyTorch's CUDA peak-memory counters after synchronization. Allocated memory counts live tensor allocations; reserved memory also includes the allocator's cached blocks. We report the maximum of each over the three calls, in GiB ($2^{30}$ bytes). Figure~\ref{fig:efficiency-bubble} uses allocated memory. These measurements include resident model tensors and inference intermediates, not just incremental activation memory. Active components reside on GPU; the benchmark does not add dynamic CPU offloading. WEAVER's unused EMA copy moves to CPU before warmup, whereas the unused modules counted above remain in the native GPU pipeline. The resulting memory costs describe these implementations, rather than minimum architectural requirements.

\clearpage
\subsection{Native Precision, Decoding, and Caches}
\label{app:resource-configurations}

Table~\ref{tab:resource-configurations} records the model-specific execution settings. FP32 and BF16 denote parameter storage; autocast denotes mixed-precision execution with the original storage retained. We preserve native differences rather than forcing a shared precision or resolution that changes the evaluated pipelines. All models produce the same final output dimensions, but internal processing costs differ.

\begingroup
\begin{table}[H]
\centering
\caption{\textbf{Native settings for the resource benchmark.} Resolution is internal height $\times$ width. Chunk lengths count retained future frames; Cosmos and DreamDojo generate a full final chunk and retain its first eight frames.}
\label{tab:resource-configurations}
\vspace{3pt}
\fontencoding{T1}\fontfamily{EBGaramond-TLF}\selectfont
\fontsize{10}{11.5}\selectfont
\setlength{\tabcolsep}{3pt}
\renewcommand{\arraystretch}{1.18}
\begin{tabularx}{\linewidth}{@{}>{\raggedright\arraybackslash}p{0.23\linewidth}>{\raggedright\arraybackslash}p{0.13\linewidth}>{\raggedright\arraybackslash}p{0.24\linewidth}>{\raggedright\arraybackslash}X@{}}
\toprule[0.9pt]
\textbf{Model} & \textbf{Resolution} & \textbf{Precision} & \textbf{Generation configuration}\\
\midrule[0.5pt]
iVideoGPT & $256\times256$ & FP32; BF16 autocast & Sampling, $T=1$, top-$k=100$; chunks $10+10+10+2$.\\[3pt]
WEAVER & $192\times320$ & FP32; BF16 autocast & Native full-memory generation, 16 denoising steps; history 2, memory 6.\\[3pt]
PersistWorld & $256\times320$ & BF16 predictor/action encoder; FP32 VAE & 25 denoising steps, guidance 2; retained chunks $10+10+10+2$.\\[3pt]
Wan2.2-TI2V-5B & $256\times320$ & FP32; BF16 autocast & UniPC, 50 steps, guidance 5, shift 5; direct 32-frame prediction.\\[3pt]
Cosmos-Predict2.5-2B & $256\times320$ & BF16 predictor and codec & 35 denoising steps, guidance 0; retained chunks $12+12+8$.\\[3pt]
DreamDojo-2B & $480\times640$ & BF16; BF16 autocast & 35 denoising steps, guidance 0; retained chunks $12+12+8$.\\[3pt]
FutureWorlds & $256\times320$ & FP32, without autocast & Beam width 4; persistent memory; 32 autoregressive frames.\\
\bottomrule[0.9pt]
\end{tabularx}
\end{table}
\endgroup

\noindent\textbf{\emph{Caching and software.}}
Input trajectories are loaded into CPU arrays before measurement; visual encoding remains inside the timed call. Language features are computed offline where used. FutureWorlds reuses the key--value cache within each frame's beam search and starts a new generation call for the next frame's retained prefix. Histories and predicted-state caches are rebuilt for each rollout, rather than reused from the warmup or another case. CUDA allocator caches remain warm between measured calls: the harness calls \texttt{empty\_cache} before warmup, not between the three timed runs. PyTorch versions are 2.9.1 for FutureWorlds, iVideoGPT, and Wan; 2.7.0+cu128 for WEAVER, Cosmos, and DreamDojo; and 2.5.1+cu124 for PersistWorld.

\noindent\textbf{\emph{Interpretation.}}
FutureWorlds uses the fewest inference parameters and the lowest observed allocated memory among these seven implementations, while achieving the best BridgeV2 LPIPS in the main comparison. It is slower than iVideoGPT and several diffusion pipelines under their native settings. The measurements therefore support a quality--model-size--memory tradeoff, not universal inference-speed superiority. Native precision, internal resolution, retained modules, and runtime versions all affect cost; three timed cases provide a small-sample operational comparison rather than a hardware-independent complexity estimate. We do not report FLOPs.

%% file: tables/resource_measurements.tex
\begingroup
\definecolor{respurple}{HTML}{544178}
\definecolor{resheader}{HTML}{E9E8F2}
\definecolor{resstripe}{HTML}{F6F7FA}
\definecolor{resours}{HTML}{EEE7F7}
\begin{table}[H]
\centering
\caption{\textbf{Native inference resource measurements on BridgeV2.} One H20 per model, batch size one, and 32 predicted frames. Latency summarizes three fixed cases after one complete warmup. Peak memory is the maximum across the three measured runs.}
\label{tab:resource-measurements}
\vspace{3pt}
\fontencoding{T1}\fontfamily{EBGaramond-TLF}\selectfont
\fontsize{10}{11.5}\selectfont
\setlength{\tabcolsep}{2pt}
\renewcommand{\arraystretch}{1.12}
\begin{tabularx}{\linewidth}{@{}p{0.27\linewidth}*{5}{>{\centering\arraybackslash}X}>{\centering\arraybackslash}p{0.15\linewidth}@{}}
\toprule[0.9pt]
\rowcolor{resheader}
\textbf{Model} & \multicolumn{2}{c}{\textbf{Parameters (B)}} & \multicolumn{2}{c}{\textbf{Peak memory (GiB)}} & \multicolumn{2}{c@{}}{\textbf{Latency (s)}}\\
\cmidrule(lr){2-3}\cmidrule(lr){4-5}\cmidrule(l){6-7}
 & Inference & Loaded & Allocated & Reserved & Median & Range\\
\midrule[0.5pt]
\rowcolor{resstripe} iVideoGPT & 0.449 & 0.449 & 5.85 & 6.53 & 4.14 & 4.12--4.16\\[2pt]
WEAVER & 0.998 & 1.134 & 7.52 & 9.02 & 6.31 & 6.29--6.31\\[2pt]
\rowcolor{resstripe} PersistWorld & 1.624 & 2.320 & 6.68 & 7.84 & 22.18 & 22.16--22.18\\[2pt]
Wan2.2-TI2V-5B & 5.754 & 5.754 & 23.06 & 23.72 & 11.18 & 11.18--11.19\\[2pt]
\rowcolor{resstripe} Cosmos-Predict2.5-2B & 2.254 & 2.254 & 5.97 & 6.85 & 12.83 & 12.77--12.83\\[2pt]
DreamDojo-2B & 2.278 & 2.278 & 12.43 & 16.26 & 47.17 & 47.11--47.19\\[2pt]
\rowcolor{resours} \textcolor{respurple}{\bfseries FutureWorlds (ours)} & 0.257 & 0.257 & 1.89 & 3.91 & 26.76 & 26.68--26.77\\[2pt]
\bottomrule[0.9pt]
\end{tabularx}
\end{table}
\endgroup

%% file: iclr2027_conference.bbl
\begin{thebibliography}{49}
\providecommand{\natexlab}[1]{#1}
\providecommand{\url}[1]{\texttt{#1}}
\expandafter\ifx\csname urlstyle\endcsname\relax
  \providecommand{\doi}[1]{doi: #1}\else
  \providecommand{\doi}{doi: \begingroup \urlstyle{rm}\Url}\fi

\bibitem[Anthony et~al.(2017)Anthony, Tian, and Barber]{anthony2017expert}
Thomas Anthony, Zheng Tian, and David Barber.
\newblock {Thinking Fast and Slow with Deep Learning and Tree Search}.
\newblock In \emph{Advances in Neural Information Processing Systems}, 2017.
\newblock URL \url{https://arxiv.org/abs/1705.08439}.

\bibitem[Babaeizadeh et~al.(2021)Babaeizadeh, Saffar, Nair, Levine, Finn, and Erhan]{babaeizadeh2021fitvid}
Mohammad Babaeizadeh, Mohammad~Taghi Saffar, Suraj Nair, Sergey Levine, Chelsea Finn, and Dumitru Erhan.
\newblock {FitVid: Overfitting in Pixel-Level Video Prediction}.
\newblock \emph{arXiv preprint arXiv:2106.13195}, 2021.
\newblock URL \url{https://arxiv.org/abs/2106.13195}.

\bibitem[Bardhan et~al.(2026)Bardhan, Drozd\'{i}k, \v{S}ivic, and Petr\'{i}k]{bardhan2026persistworld}
Jai Bardhan, Patrik Drozd\'{i}k, Josef \v{S}ivic, and Vladim\'{i}r Petr\'{i}k.
\newblock {PersistWorld: Stabilizing Multi-step Robot World Model Rollouts via Reinforcement Learning}.
\newblock \emph{arXiv preprint}, 2026.
\newblock URL \url{https://jaibardhan.com/persistworld/}.

\bibitem[Bengio et~al.(2015)Bengio, Vinyals, Jaitly, and Shazeer]{bengio2015sampling}
Samy Bengio, Oriol Vinyals, Navdeep Jaitly, and Noam Shazeer.
\newblock {Scheduled Sampling for Sequence Prediction with Recurrent Neural Networks}.
\newblock In \emph{Advances in Neural Information Processing Systems}, 2015.
\newblock URL \url{https://arxiv.org/abs/1506.03099}.

\bibitem[Black et~al.(2024)Black, Janner, Du, Kostrikov, and Levine]{black2024ddpo}
Kevin Black, Michael Janner, Yilun Du, Ilya Kostrikov, and Sergey Levine.
\newblock {Training Diffusion Models with Reinforcement Learning}.
\newblock In \emph{International Conference on Learning Representations}, 2024.
\newblock URL \url{https://arxiv.org/abs/2305.13301}.

\bibitem[Brohan et~al.(2023)Brohan, Brown, Carbajal, Chebotar, Dabis, Finn, Gopalakrishnan, Hausman, Herzog, Hsu, Ibarz, Ichter, Irpan, Jackson, Jesmonth, Joshi, Julian, Kalashnikov, Kuang, Leal, Lee, Levine, Lu, Malla, Manjunath, Mordatch, Nachum, Parada, Peralta, Perez, Pertsch, Quiambao, Rao, Ryoo, Salazar, Sanketi, Sayed, Singh, Sontakke, Stone, Tan, Tran, Vanhoucke, Vega, Vuong, Xia, Xiao, Xu, Xu, Yu, and Zitkovich]{brohan2023rt1}
Anthony Brohan, Noah Brown, Justice Carbajal, Yevgen Chebotar, Joseph Dabis, Chelsea Finn, Keerthana Gopalakrishnan, Karol Hausman, Alex Herzog, Jasmine Hsu, Julian Ibarz, Brian Ichter, Alex Irpan, Tomas Jackson, Sally Jesmonth, Nikhil~J Joshi, Ryan Julian, Dmitry Kalashnikov, Yuheng Kuang, Isabel Leal, Kuang-Huei Lee, Sergey Levine, Yao Lu, Utsav Malla, Deeksha Manjunath, Igor Mordatch, Ofir Nachum, Carolina Parada, Jodilyn Peralta, Emily Perez, Karl Pertsch, Jornell Quiambao, Kanishka Rao, Michael Ryoo, Grecia Salazar, Pannag Sanketi, Kevin Sayed, Jaspiar Singh, Sumedh Sontakke, Austin Stone, Clayton Tan, Huong Tran, Vincent Vanhoucke, Steve Vega, Quan Vuong, Fei Xia, Ted Xiao, Peng Xu, Sichun Xu, Tianhe Yu, and Brianna Zitkovich.
\newblock {RT-1: Robotics Transformer for Real-World Control at Scale}.
\newblock In \emph{Robotics: Science and Systems}, 2023.
\newblock URL \url{https://roboticsproceedings.org/rss19/p025.html}.

\bibitem[Dai et~al.(2019)Dai, Yang, Yang, Carbonell, Le, and Salakhutdinov]{dai2019transformerxl}
Zihang Dai, Zhilin Yang, Yiming Yang, Jaime~G. Carbonell, Quoc Le, and Ruslan Salakhutdinov.
\newblock {Transformer-XL: Attentive Language Models beyond a Fixed-Length Context}.
\newblock In \emph{Proceedings of the 57th Annual Meeting of the Association for Computational Linguistics}, pp.\  2978--2988, 2019.
\newblock \doi{10.18653/v1/P19-1285}.
\newblock URL \url{https://aclanthology.org/P19-1285/}.

\bibitem[Dong et~al.(2026)Dong, Jia, Finn, and Sadigh]{dong2026qwm}
Perry Dong, Yueru Jia, Chelsea Finn, and Dorsa Sadigh.
\newblock {Q-Learning With World Models}.
\newblock \emph{arXiv preprint arXiv:2608.17163}, 2026.
\newblock URL \url{https://arxiv.org/abs/2608.17163}.

\bibitem[Ebert et~al.(2018)Ebert, Finn, Dasari, Xie, Lee, and Levine]{ebert2018visual}
Frederik Ebert, Chelsea Finn, Sudeep Dasari, Annie Xie, Alex Lee, and Sergey Levine.
\newblock {Visual Foresight: Model-Based Deep Reinforcement Learning for Vision-Based Robotic Control}.
\newblock \emph{arXiv preprint arXiv:1812.00568}, 2018.
\newblock URL \url{https://arxiv.org/abs/1812.00568}.

\bibitem[Finn et~al.(2016)Finn, Goodfellow, and Levine]{finn2016physical}
Chelsea Finn, Ian Goodfellow, and Sergey Levine.
\newblock {Unsupervised Learning for Physical Interaction through Video Prediction}.
\newblock In \emph{Advances in Neural Information Processing Systems}, 2016.
\newblock URL \url{https://arxiv.org/abs/1605.07157}.

\bibitem[Gao et~al.(2026)Gao, Liang, Zheng, Malik, Ye, Yu, et~al.]{gao2026dreamdojo}
Shenyuan Gao, William Liang, Kaiyuan Zheng, Ayaan Malik, Seonghyeon Ye, Sihyun Yu, et~al.
\newblock {DreamDojo}: A generalist robot world model from large-scale human videos.
\newblock \emph{arXiv preprint arXiv:2602.06949}, 2026.
\newblock URL \url{https://arxiv.org/abs/2602.06949}.

\bibitem[{GigaWorld Team} et~al.(2026){GigaWorld Team}, Ma, Wang, Li, Ni, Li, Huang, Zhao, Li, Li, Liu, Lu, Deng, Yu, Xu, Zhou, Xu, Chen, Wang, Tian, Wang, Chang, Zhou, Ye, Wu, Wu, and Zhu]{gigaworld2026}
{GigaWorld Team}, Angyuan Ma, Boyuan Wang, Bohan Li, Chaojun Ni, Guo Li, Guan Huang, Guosheng Zhao, Hao Li, Hengtao Li, Jingyu Liu, Jiwen Lu, Qiuping Deng, Tingdong Yu, Xuancheng Xu, Xinyu Zhou, Xiuwei Xu, Xinze Chen, Xiaofeng Wang, Xiaoyu Tian, Yang Wang, Yifan Chang, Yukun Zhou, Yun Ye, Zhenyu Wu, Zhanqian Wu, and Zheng Zhu.
\newblock {GigaWorld-1: A Roadmap to Build World Models for Robot Policy Evaluation}.
\newblock \emph{arXiv preprint arXiv:2607.02642}, 2026.
\newblock URL \url{https://arxiv.org/abs/2607.02642}.

\bibitem[Guo et~al.(2026)Guo, Shi, Chen, and Finn]{guo2026ctrlworld}
Yanjiang Guo, Lucy~Xiaoyang Shi, Jianyu Chen, and Chelsea Finn.
\newblock {Ctrl-World: A Controllable Generative World Model for Robot Manipulation}.
\newblock In \emph{International Conference on Learning Representations}, 2026.
\newblock URL \url{https://arxiv.org/abs/2510.10125}.

\bibitem[Hafner et~al.(2019)Hafner, Lillicrap, Fischer, Villegas, Ha, Lee, and Davidson]{hafner2019planet}
Danijar Hafner, Timothy Lillicrap, Ian Fischer, Ruben Villegas, David Ha, Honglak Lee, and James Davidson.
\newblock {Learning Latent Dynamics for Planning from Pixels}.
\newblock In \emph{International Conference on Machine Learning}, pp.\  2555--2565, 2019.
\newblock URL \url{https://proceedings.mlr.press/v97/hafner19a.html}.

\bibitem[Hafner et~al.(2020)Hafner, Lillicrap, Ba, and Norouzi]{hafner2020dreamer}
Danijar Hafner, Timothy Lillicrap, Jimmy Ba, and Mohammad Norouzi.
\newblock {Dream to Control: Learning Behaviors by Latent Imagination}.
\newblock In \emph{International Conference on Learning Representations}, 2020.
\newblock URL \url{https://arxiv.org/abs/1912.01603}.

\bibitem[Henschel et~al.(2025)Henschel, Khachatryan, Poghosyan, Hayrapetyan, Tadevosyan, Wang, Navasardyan, and Shi]{henschel2025streamingt2v}
Roberto Henschel, Levon Khachatryan, Hayk Poghosyan, Daniil Hayrapetyan, Vahram Tadevosyan, Zhangyang Wang, Shant Navasardyan, and Humphrey Shi.
\newblock {StreamingT2V: Consistent, Dynamic, and Extendable Long Video Generation from Text}.
\newblock In \emph{Proceedings of the IEEE/CVF Conference on Computer Vision and Pattern Recognition}, pp.\  2568--2577, 2025.
\newblock URL \url{https://arxiv.org/abs/2403.14773}.

\bibitem[Holtzman et~al.(2020)Holtzman, Buys, Du, Forbes, and Choi]{holtzman2020degeneration}
Ari Holtzman, Jan Buys, Li~Du, Maxwell Forbes, and Yejin Choi.
\newblock {The Curious Case of Neural Text Degeneration}.
\newblock In \emph{International Conference on Learning Representations}, 2020.
\newblock URL \url{https://arxiv.org/abs/1904.09751}.

\bibitem[Hu et~al.(2021)Hu, Shen, Wallis, Allen-Zhu, Li, Wang, Wang, and Chen]{hu2021lora}
Edward~J. Hu, Yelong Shen, Phillip Wallis, Zeyuan Allen-Zhu, Yuanzhi Li, Shean Wang, Lu~Wang, and Weizhu Chen.
\newblock {LoRA}: Low-rank adaptation of large language models.
\newblock \emph{arXiv preprint arXiv:2106.09685}, 2021.
\newblock URL \url{https://arxiv.org/abs/2106.09685}.

\bibitem[Huang et~al.(2026)Huang, Wu, Zhou, Miao, and Long]{huang2026vid2world}
Siqiao Huang, Jialong Wu, Qixing Zhou, Shangchen Miao, and Mingsheng Long.
\newblock {Vid2World: Crafting Video Diffusion Models to Interactive World Models}.
\newblock In \emph{International Conference on Learning Representations}, 2026.
\newblock URL \url{https://arxiv.org/abs/2505.14357}.

\bibitem[Jain et~al.(2026)Jain, Wu, Farebrother, Swamy, and Bajcsy]{jain2026weaver}
Arnav~Kumar Jain, Yilin Wu, Jesse Farebrother, Gokul Swamy, and Andrea Bajcsy.
\newblock {WEAVER, Better, Faster, Longer: An Effective World Model for Robotic Manipulation}.
\newblock \emph{arXiv preprint arXiv:2606.13672}, 2026.
\newblock URL \url{https://arxiv.org/abs/2606.13672}.

\bibitem[Lee et~al.(2025)Lee, Huh, Su, and Yu]{lee2025bgrpo}
Jaeha Lee, Gio Huh, Ning Su, and Tony~Yue Yu.
\newblock {Discovering Hidden Algebraic Structures via Transformers with Rank-Aware Beam GRPO}.
\newblock \emph{arXiv preprint arXiv:2508.15766}, 2025.
\newblock URL \url{https://arxiv.org/abs/2508.15766}.

\bibitem[Lu et~al.(2026)Lu, Chen, Ye, and Du]{lu2026driftworld}
Susie Lu, Haonan Chen, Weirui Ye, and Yilun Du.
\newblock {DriftWorld: Fast World Modeling through Drifting}.
\newblock \emph{arXiv preprint arXiv:2607.15065}, 2026.
\newblock URL \url{https://arxiv.org/abs/2607.15065}.

\bibitem[Mentzer et~al.(2024)Mentzer, Minnen, Agustsson, and Tschannen]{mentzer2024fsq}
Fabian Mentzer, David Minnen, Eirikur Agustsson, and Michael Tschannen.
\newblock Finite scalar quantization: {VQ-VAE} made simple.
\newblock In \emph{International Conference on Learning Representations}, 2024.
\newblock URL \url{https://arxiv.org/abs/2309.15505}.

\bibitem[Nasiriany et~al.(2026)Nasiriany, Nasiriany, Maddukuri, and Zhu]{nasiriany2026robocasa365}
Soroush Nasiriany, Sepehr Nasiriany, Abhiram Maddukuri, and Yuke Zhu.
\newblock {RoboCasa365: A Large-Scale Simulation Framework for Training and Benchmarking Generalist Robots}.
\newblock In \emph{International Conference on Learning Representations}, 2026.
\newblock URL \url{https://robocasa.ai/assets/robocasa365_iclr26.pdf}.

\bibitem[{NVIDIA}(2025)]{nvidia2025cosmos25}
{NVIDIA}.
\newblock {World Simulation with Video Foundation Models for Physical AI}.
\newblock \emph{arXiv preprint arXiv:2511.00062}, 2025.
\newblock URL \url{https://arxiv.org/abs/2511.00062}.

\bibitem[Raffel et~al.(2020)Raffel, Shazeer, Roberts, Lee, Narang, Matena, Zhou, Li, and Liu]{raffel2020t5}
Colin Raffel, Noam Shazeer, Adam Roberts, Katherine Lee, Sharan Narang, Michael Matena, Yanqi Zhou, Wei Li, and Peter~J. Liu.
\newblock Exploring the limits of transfer learning with a unified text-to-text transformer.
\newblock \emph{Journal of Machine Learning Research}, 21\penalty0 (140):\penalty0 1--67, 2020.
\newblock URL \url{https://jmlr.org/papers/v21/20-074.html}.

\bibitem[Ranzato et~al.(2016)Ranzato, Chopra, Auli, and Zaremba]{ranzato2016sequence}
Marc'Aurelio Ranzato, Sumit Chopra, Michael Auli, and Wojciech Zaremba.
\newblock {Sequence Level Training with Recurrent Neural Networks}.
\newblock In \emph{International Conference on Learning Representations}, 2016.
\newblock URL \url{https://arxiv.org/abs/1511.06732}.

\bibitem[Rennie et~al.(2017)Rennie, Marcheret, Mroueh, Ross, and Goel]{rennie2017scst}
Steven~J. Rennie, Etienne Marcheret, Youssef Mroueh, Jerret Ross, and Vaibhava Goel.
\newblock {Self-Critical Sequence Training for Image Captioning}.
\newblock In \emph{Proceedings of the IEEE Conference on Computer Vision and Pattern Recognition}, 2017.
\newblock URL \url{https://arxiv.org/abs/1612.00563}.

\bibitem[Ross et~al.(2011)Ross, Gordon, and Bagnell]{ross2011dagger}
Stephane Ross, Geoffrey Gordon, and Drew Bagnell.
\newblock {A Reduction of Imitation Learning and Structured Prediction to No-Regret Online Learning}.
\newblock In \emph{International Conference on Artificial Intelligence and Statistics}, pp.\  627--635, 2011.
\newblock URL \url{https://proceedings.mlr.press/v15/ross11a.html}.

\bibitem[Schulman et~al.(2017)Schulman, Wolski, Dhariwal, Radford, and Klimov]{schulman2017ppo}
John Schulman, Filip Wolski, Prafulla Dhariwal, Alec Radford, and Oleg Klimov.
\newblock {Proximal Policy Optimization Algorithms}.
\newblock \emph{arXiv preprint arXiv:1707.06347}, 2017.
\newblock URL \url{https://arxiv.org/abs/1707.06347}.

\bibitem[Shang et~al.(2026)Shang, Li, Ma, Su, Jin, Wang, Jin, Zhang, Tang, Su, et~al.]{shang2026worldarena}
Yu~Shang, Zhuohang Li, Yiding Ma, Weikang Su, Xin Jin, Ziyou Wang, Lei Jin, Xin Zhang, Yinzhou Tang, Haisheng Su, et~al.
\newblock Worldarena: A unified benchmark for evaluating perception and functional utility of embodied world models.
\newblock \emph{arXiv preprint arXiv:2602.08971}, 2026.

\bibitem[Shao et~al.(2024)Shao, Wang, Zhu, Xu, Song, Bi, Zhang, Zhang, Li, Wu, and Guo]{shao2024deepseekmath}
Zhihong Shao, Peiyi Wang, Qihao Zhu, Runxin Xu, Junxiao Song, Xiao Bi, Haowei Zhang, Mingchuan Zhang, Y.~K. Li, Y.~Wu, and Daya Guo.
\newblock {DeepSeekMath: Pushing the Limits of Mathematical Reasoning in Open Language Models}.
\newblock \emph{arXiv preprint arXiv:2402.03300}, 2024.
\newblock URL \url{https://arxiv.org/abs/2402.03300}.

\bibitem[{Team Wan} et~al.(2025){Team Wan}, Wang, Ai, Wen, Mao, Xie, et~al.]{wan2025wan}
{Team Wan}, Ang Wang, Baole Ai, Bin Wen, Chaojie Mao, Chen-Wei Xie, et~al.
\newblock {Wan}: Open and advanced large-scale video generative models.
\newblock \emph{arXiv preprint arXiv:2503.20314}, 2025.
\newblock URL \url{https://arxiv.org/abs/2503.20314}.

\bibitem[van~den Oord et~al.(2017)van~den Oord, Vinyals, and Kavukcuoglu]{oord2017vqvae}
Aaron van~den Oord, Oriol Vinyals, and Koray Kavukcuoglu.
\newblock {Neural Discrete Representation Learning}.
\newblock In \emph{Advances in Neural Information Processing Systems}, 2017.
\newblock URL \url{https://arxiv.org/abs/1711.00937}.

\bibitem[Vijayakumar et~al.(2018)Vijayakumar, Cogswell, Selvaraju, Sun, Lee, Crandall, and Batra]{vijayakumar2018diverse}
Ashwin~K Vijayakumar, Michael Cogswell, Ramprasath~R. Selvaraju, Qing Sun, Stefan Lee, David Crandall, and Dhruv Batra.
\newblock {Diverse Beam Search for Improved Description of Complex Scenes}.
\newblock In \emph{Proceedings of the AAAI Conference on Artificial Intelligence}, 2018.
\newblock URL \url{https://ojs.aaai.org/index.php/AAAI/article/view/12340}.

\bibitem[Walke et~al.(2023)Walke, Black, Zhao, Vuong, Zheng, Hansen-Estruch, He, Myers, Kim, Du, Lee, Fang, Finn, and Levine]{walke2023bridge}
Homer~Rich Walke, Kevin Black, Tony~Z. Zhao, Quan Vuong, Chongyi Zheng, Philippe Hansen-Estruch, Andre~Wang He, Vivek Myers, Moo~Jin Kim, Max Du, Abraham Lee, Kuan Fang, Chelsea Finn, and Sergey Levine.
\newblock {BridgeData V2: A Dataset for Robot Learning at Scale}.
\newblock In \emph{Conference on Robot Learning}, pp.\  1723--1736, 2023.
\newblock URL \url{https://proceedings.mlr.press/v229/walke23a.html}.

\bibitem[Wu et~al.(2024)Wu, Yin, Feng, He, Li, Hao, and Long]{wu2024ivideogpt}
Jialong Wu, Shaofeng Yin, Ningya Feng, Xu~He, Dong Li, Jianye Hao, and Mingsheng Long.
\newblock {iVideoGPT: Interactive VideoGPTs are Scalable World Models}.
\newblock In \emph{Advances in Neural Information Processing Systems}, 2024.
\newblock URL \url{https://arxiv.org/abs/2405.15223}.

\bibitem[Wu et~al.(2025{\natexlab{a}})Wu, Yin, Feng, and Long]{wu2025rlvr}
Jialong Wu, Shaofeng Yin, Ningya Feng, and Mingsheng Long.
\newblock {RLVR-World: Training World Models with Reinforcement Learning}.
\newblock In \emph{Advances in Neural Information Processing Systems}, 2025{\natexlab{a}}.
\newblock URL \url{https://arxiv.org/abs/2505.13934}.

\bibitem[Wu et~al.(2025{\natexlab{b}})Wu, Yang, Po, Xu, Liu, Lin, and Wetzstein]{wu2025spatialmemory}
Tong Wu, Shuai Yang, Ryan Po, Yinghao Xu, Ziwei Liu, Dahua Lin, and Gordon Wetzstein.
\newblock {Video World Models with Long-term Spatial Memory}.
\newblock \emph{arXiv preprint arXiv:2506.05284}, 2025{\natexlab{b}}.
\newblock URL \url{https://arxiv.org/abs/2506.05284}.

\bibitem[Xiao et~al.(2024)Xiao, Tian, Chen, Han, and Lewis]{xiao2024streamingllm}
Guangxuan Xiao, Yuandong Tian, Beidi Chen, Song Han, and Mike Lewis.
\newblock {Efficient Streaming Language Models with Attention Sinks}.
\newblock In \emph{International Conference on Learning Representations}, 2024.
\newblock URL \url{https://arxiv.org/abs/2309.17453}.

\bibitem[Yan et~al.(2021)Yan, Zhang, Abbeel, and Srinivas]{yan2021videogpt}
Wilson Yan, Yunzhi Zhang, Pieter Abbeel, and Aravind Srinivas.
\newblock {VideoGPT: Video Generation using VQ-VAE and Transformers}.
\newblock \emph{arXiv preprint arXiv:2104.10157}, 2021.
\newblock URL \url{https://arxiv.org/abs/2104.10157}.

\bibitem[Yang et~al.(2024)Yang, Du, Ghasemipour, Tompson, Kaelbling, Schuurmans, and Abbeel]{yang2024unisim}
Sherry Yang, Yilun Du, Seyed Ghasemipour, Jonathan Tompson, Leslie Kaelbling, Dale Schuurmans, and Pieter Abbeel.
\newblock {Learning Interactive Real-World Simulators}.
\newblock In \emph{International Conference on Learning Representations}, 2024.
\newblock URL \url{https://arxiv.org/abs/2310.06114}.

\bibitem[Ye et~al.(2025)Ye, He, Yang, and Bian]{ye2025rlir}
Yang Ye, Tianyu He, Shuo Yang, and Jiang Bian.
\newblock {Reinforcement Learning with Inverse Rewards for World Model Post-training}.
\newblock \emph{arXiv preprint arXiv:2509.23958}, 2025.
\newblock URL \url{https://arxiv.org/abs/2509.23958}.

\bibitem[Yin et~al.(2026)Yin, Mei, Zheng, Yamane, Wang, Sceats, Bateman, Zha, Badithela, Shorinwa, and Majumdar]{yin2026playworld}
Tenny Yin, Zhiting Mei, Zhonghe Zheng, Miyu Yamane, David Wang, Jade Sceats, Samuel~M. Bateman, Lihan Zha, Apurva Badithela, Ola Shorinwa, and Anirudha Majumdar.
\newblock {PlayWorld: Learning Robot World Models from Autonomous Play}.
\newblock \emph{arXiv preprint arXiv:2603.09030}, 2026.
\newblock URL \url{https://arxiv.org/abs/2603.09030}.

\bibitem[Zhang \& Agrawala(2025)Zhang and Agrawala]{zhang2025framepack}
Lvmin Zhang and Maneesh Agrawala.
\newblock {Packing Input Frame Context in Next-Frame Prediction Models for Video Generation}.
\newblock \emph{arXiv preprint arXiv:2504.12626}, 2025.
\newblock URL \url{https://arxiv.org/abs/2504.12626}.

\bibitem[Zhang et~al.(2018)Zhang, Isola, Efros, Shechtman, and Wang]{zhang2018lpips}
Richard Zhang, Phillip Isola, Alexei~A. Efros, Eli Shechtman, and Oliver Wang.
\newblock {The Unreasonable Effectiveness of Deep Features as a Perceptual Metric}.
\newblock In \emph{IEEE Conference on Computer Vision and Pattern Recognition}, 2018.
\newblock URL \url{https://arxiv.org/abs/1801.03924}.

\bibitem[Zhao et~al.(2023)Zhao, Bai, Rao, Zhou, and Lu]{zhao2023unipc}
Wenliang Zhao, Lujia Bai, Yongming Rao, Jie Zhou, and Jiwen Lu.
\newblock {UniPC}: A unified predictor-corrector framework for fast sampling of diffusion models.
\newblock In \emph{Advances in Neural Information Processing Systems}, 2023.
\newblock URL \url{https://arxiv.org/abs/2302.04867}.

\bibitem[Zheng et~al.(2026)Zheng, Yu, Peng, Shi, Li, Zhang, Li, Wang, Lu, and Jia]{zheng2026memworld}
Zirui Zheng, Jiaqian Yu, Xiongfeng Peng, Jun Shi, Mingyi Li, Chao Zhang, Weiming Li, Dong Wang, Huchuan Lu, and Xu~Jia.
\newblock {Mem-World: Memory-Augmented Action-Conditioned World Models for Persistent Robot Manipulation}.
\newblock \emph{arXiv preprint arXiv:2606.18960}, 2026.
\newblock URL \url{https://arxiv.org/abs/2606.18960}.

\bibitem[Zhu et~al.(2025)]{zhu2025irasim}
Fangqi Zhu et~al.
\newblock {IRASim: A Fine-Grained World Model for Robot Manipulation}.
\newblock \emph{arXiv preprint arXiv:2406.14540}, 2025.
\newblock URL \url{https://arxiv.org/abs/2406.14540}.

\end{thebibliography}
